\documentclass[]{arxivCOLT2019} 
\usepackage[shortlabels]{enumitem}

\usepackage{amssymb}

\usepackage[utf8]{inputenc}

\usepackage{times}
\usepackage{hyperref}
\usepackage{bm}
\usepackage{amsfonts}

\usepackage{graphicx,wrapfig}

\usepackage{subfiles}

\usepackage{enumitem}
\usepackage{verbatim}
\newtheorem{lem}[theorem]{Lemma}
\newtheorem{assumption}{Assumption}

\newtheorem{defn}{Definition}
\numberwithin{defn}{section}
\newtheorem{prop}[theorem]{Proposition}

\usepackage{float}
\usepackage{epigraph}
\usepackage{dirtytalk}
\usepackage{mathtools}
\usepackage{thmtools}
\usepackage{thm-restate}
\usepackage{booktabs}
\usepackage{natbib}
\usepackage{enumitem}
\usepackage{xparse}
\usepackage{thm-restate}
\usepackage{bbm}

\newcommand{\R}{\mathbb{R}}
\newcommand{\N}{\mathbb{N}}

\newcommand{\X}{\mathcal{X}}

\newcommand{\risk}{\mathcal{R}}
\newcommand{\Prob}{\mathbb{P}}
\newcommand{\E}{\mathbb{E}}

\newcommand{\one}{\mathbbm{1}}

\newcommand{\est}[1][]{\hat{#1}}

\newcommand{\marginalDistribution}{\mu}
\newcommand{\regressionFunction}{\eta}

\newcommand{\Y}{\mathcal{Y}}

\newcommand{\dist}{\rho}

\newcommand{\oracle}{\decisionRule^*}
\newcommand{\sample}{\mathcal{D}}

\newcommand{\setOfDecisionRules}{\mathcal{F}(\X,\Y)}
\newcommand{\decisionRule}{\phi}
\newcommand{\classifier}{\hat{\phi}_n}

\newcommand{\distributionClass}{\mathcal{P}}

\newcommand{\ExpectationOverNoisySamples}{\E_{\text{corr}}^{\otimes n}}
\newcommand{\ExpectationOverCleanSamples}{\E_{\text{clean}}^{\otimes n}}
\newcommand{\ProbOverNoisySamples}{\Prob_{\text{corr}}^{\otimes n}}
\newcommand{\ProbOverCleanSamples}{\Prob_{\text{clean}}^{\otimes n}}

\newcommand{\cleanDistribution}{\Prob_{\text{clean}}}
\newcommand{\noisyDistribution}{\Prob_{\text{corr}}}

\newcommand{\noisy}[1][]{
    \ifthenelse{\equal{#1}{Y}}
    {\tilde{Y}}
    {{#1}_{\text{corr}}}
    } 
\newcommand{\noisySample}{\noisy[\sample]}

\newcommand{\noisyRegressionFunction}{\noisy[\regressionFunction]}
\newcommand{\estNoisyRegressionFunction}{\est[\regressionFunction]\noisy[]}

\newcommand{\probFlipFrom}[1][y]{\pi_{#1}}
\newcommand{\probFlipZeroToOne}{\probFlipFrom[0]}
\newcommand{\probFlipOneToZero}{\probFlipFrom[1]}

\newcommand{\estProbFlipZeroToOne}{\est[\pi]_0}
\newcommand{\estProbFlipOneToZero}{\est[\pi]_1}

\newcommand{\marginExponent}{\alpha}
\newcommand{\marginConstant}{C_{\marginExponent}}

\newcommand{\holderExponent}{\beta}
\newcommand{\holderConstant}{C_{\holderExponent}}

\newcommand{\dimensionExponent}{d}
\newcommand{\dimensionRadius}{r_d}
\newcommand{\densityFunction}{\omega_{\marginalDistribution}}

\newcommand{\maxAggregateFlippingProbability}{\nu_{\max}}

\newcommand{\tailsExponent}{\gamma}
\newcommand{\tailsConstant}{C_{\tailsExponent}}
\newcommand{\tailsThreshold}{t_{\tailsExponent}}

\newcommand{\fastLimitsExponent}{\tau}
\newcommand{\fastLimitsConstant}{C_{\fastLimitsExponent}}
\newcommand{\fastLimitsThreshold}{t_{\fastLimitsExponent}}

\newcommand{\suppMarginalDistribution}{\X_{\marginalDistribution}}

\newcommand{\kullbackLeiblerDivergence}{D_{KL}}

\newcommand*{\QEDA}{\hfill\ensuremath{\blacksquare}}

\usepackage{enumitem}
\setenumerate{noitemsep,topsep=1pt,parsep=1pt,partopsep=1pt}

\title[Classification with label noise on non-compact feature spaces]{Classification with unknown class-conditional label noise on non-compact feature spaces}

\coltauthor{%
 \Name{Henry W. J. Reeve} \Email{henrywjreeve@gmail.com}\\
 \addr{University of Birmingham, UK}\\
 \Name{Ata Kab\'an} \Email{ata.x.kaban@gmail.com}\\
 \addr{University of Birmingham, UK}
}
\begin{document}

\maketitle

\begin{abstract}%
We investigate the problem of classification in the presence of unknown class-conditional label noise in which the labels observed by the learner have been corrupted with some unknown class dependent probability. In order to obtain finite sample rates, previous approaches to classification with unknown class-conditional label noise have required that the regression function is close to its extrema on sets of large measure. We shall consider this problem in the setting of non-compact metric spaces, where the regression function need not attain its extrema.

In this setting we determine the minimax optimal learning rates (up to logarithmic factors). The rate displays interesting threshold behaviour: When the regression function approaches its extrema at a sufficient rate, the optimal learning rates are of the same order as those obtained in the label-noise free setting. If the regression function approaches its extrema more gradually then classification performance necessarily degrades. In addition, we present an adaptive algorithm which attains these rates without prior knowledge of either the distributional parameters or the local density. This identifies for the first time a scenario in which finite sample rates are achievable in the label noise setting, but they differ from the optimal rates without label noise.

\end{abstract}


\section{Introduction}
In this paper we investigate the problem of classification with unknown class-conditional label noise on non-compact metric spaces. We determine minimax optimal learning rates which reveal an interesting dependency upon the behaviour of the regression function in the tails of the distribution. 

Classification with label noise is a problem of great practical significance in machine learning. Whilst it is typically assumed that the train and test distributions are one and the same, it is often the case that the labels in the training data have been \emph{corrupted} with some unknown probability \cite[]{Frenay1}. We shall focus on the problem of class-conditional label noise, where the label noise depends on the class label \citep{bootkrajang2014learning}. This has numerous applications including learning from positive and unlabelled data \citep{ElkanNoto,li2019positive} and nuclear particle classification \cite[]{natarajan2013learning,blanchard2016classification}. Learning from class-conditional label noise is complicated by the fact that the optimal decision boundary will typically differ between test and train distributions. This effect can be accommodated for if the learner has prior knowledge of the label noise probabilities (the probability of flipping from one class to another) \citep{natarajan2013learning}. Unfortunately, this is rarely the case in practice.

The seminal work of \cite{scott2013classification} showed that the label noise probabilities may be consistently estimated from the data, under the mutual irreducibility assumption, which is equivalent to the assumption that the regression function $\regressionFunction$ has infimum zero and supremum one \cite[]{menon2015learning}. Without further assumptions the rate of convergence may be arbitrarily slow \cite[]{blanchard2010semi,scott2013classification,blanchard2016classification}. However, \cite{scott2015rate} demonstrated that a finite sample rate of order $O(1/\sqrt{n})$ may be obtained provided that the following strong irreducibility condition holds: There exists a family of sets $\mathcal{S}$ of finite VC dimension (eg. the set of metric balls in $\R^{\dimensionExponent}$), such that for a pair of sets $S_0$, $S_1 \in \mathcal{S}$ of positive measure, the regression function $\regressionFunction$ is uniformly zero on $S_0$ and uniformly one on $S_1$. Finite sample rates have also been obtained by \cite{reeveKaban19a} for the robust $k$-nearest neighbour classifier of \cite{gao2018}, with a strong uniform smoothness condition, in conjunction with the mutual irreducibility condition of \cite{scott2013classification}. In both cases, the learning rates for classification with unknown-class conditional label noise match the optimal rates for the corresponding label noise free setting, up to logarithmic terms.  This motivates the question of whether there are scenarios in which finite sample rates are achievable in the label noise setting, yet the rates differ from the optimal rates without label noise?

\begin{comment}

In both cases, the regression function is close to its extrema on sets of large measure. These conditions imply that the statistical difficulty of estimating the label noise probabilities is dominated by the difficulty of the classification problem. Consequently, the finite sample rates for classification with unknown label noise match the optimal rates for the corresponding label noise free setting, up to logarithmic terms.  This motivates the question of whether or not there are scenarios in which finite sample rates with label noise are achievable, yet they differ from the corresponding label noise free rate. This occurs, for example, when the class-conditional distributions are mixtures of multivariate Gaussians. However, such settings lie in stark contrast to the assumptions of \cite{scott2013classification}.
\end{comment}
In this work we focus on a flexible non-parametric setting which incorporates various natural examples where the marginal distribution is supported on a non-compact metric space. We will make a flexible tail assumption, due to \cite{gadat2016}, which controls the decay of the measure of regions of the feature space where the density is below a given threshold. This avoids the common yet restrictive assumption that the density is bounded uniformly from below or the assumption of finite covering dimension \citep{audibert2007fast}. For non-compact metric spaces it is natural to consider settings where the regression function never attains its infimum and supremum, and instead approaches these values asymptotically, in the tails of the distribution. This occurs, for example, when the class-conditional distributions are mixtures of multivariate Gaussians. In this work we explore the relationship between the rate at which the regression function approaches its extrema and the optimal learning rates. Our contributions are as follows:
\begin{itemize}
\item We determine the minimax optimal learning rate (up to logarithmic factors) for classification in the presence of unknown class-conditional label noise on non-compact metric spaces (Theorems \ref{mainMinimaxThmLowerBound} and \ref{mainUpperBoundThm}). The rate displays interesting threshold behaviour: When the regression function approaches its extrema at a sufficient rate, the optimal learning rates are of the same order as those obtained by \cite{gadat2016} in the label-noise free setting. If the regression function approaches its extrema more gradually then classification performance necessarily degrades. This identifies, for the first time, a scenario in which finite sample rates are achievable in the label noise setting, but they differ from the rates achievable without label noise.
\item We present an algorithm for classification with unknown class-conditional label noise on non-compact metric spaces. The algorithm is straightforward to implement and adaptive, in the sense that it does not require any prior knowledge of the distributional parameters or the local density. A high probability upper bound is proved which demonstrates that the performance of the algorithm is optimal, up to logarithmic factors (Theorem \ref{mainUpperBoundThm}).
\item As a byproduct of our analysis, we introduce a simple and adaptive method for estimating the maximum of a function on a non-compact domain. A high probability bound on its performance is given, with a rate governed by the local density at the maximum, if the maximum is attained, or the rate at which the function approaches its maximum otherwise (Theorem \ref{funcMaxThm}). 
\end{itemize}

We begin by formalising the statistical setting in Section \ref{statisticalSettingSec}. We then present our minimax lower bound in Section \ref{mainLBSec}. In Section \ref{mainAlgoSec} we introduce an adaptive algorithm with a high probability upper bound. Formal proofs may be found within the Appendix.

\section{The statistical setting}\label{statisticalSettingSec}

We consider the problem of binary classification in metric spaces with class-conditional label noise. Suppose we have a metric space $\left(\X,\dist\right)$, a set of possible labels $\Y= \{0,1\}$, and a distribution $\Prob$ over triples $(X,Y,\noisy[Y])\in \X\times\Y^2$, consisting of a feature vector $X \in \X$, a true class label $Y \in \Y$ and a corrupted label $\noisy[Y] \in \Y$, which may be distinct from $Y$. We let $\cleanDistribution$ denote the marginal distribution over $(X,Y)$ and let $\noisyDistribution$ denote the marginal distribution over $(X,\noisy[Y])$. Let $\setOfDecisionRules$ denote the set of all decision rules, which are Borel measureable mappings $\decisionRule:\X\rightarrow \Y$. The goal of learner is to determine a decision rule $\decisionRule \in \setOfDecisionRules$ which minimises the risk
\begin{align*}
\risk\left(\decisionRule\right):=\cleanDistribution\left[\decisionRule(X)\neq Y\right] = \int \left(\decisionRule(x)  \left(1-\regressionFunction(x)\right) +\left(1-\decisionRule(x)\right)\regressionFunction(x)\right) d\marginalDistribution(x) 
\end{align*}
 Here $\regressionFunction:\X \rightarrow [0,1]$ denotes the regression function $\regressionFunction(x):= \cleanDistribution\left[Y=1|X=x\right]$ and $\marginalDistribution$ denotes the marginal distribution over $X$. The risk is minimised by the Bayes decision rule $\oracle\in \setOfDecisionRules$ defined by $\oracle(x):= \one\left\lbrace \regressionFunction(x) \geq \frac{1}{2}\right\rbrace$. Since $\regressionFunction$ is unobserved, the learner must rely upon data. We assume that the learner has access to a corrupted sample $\noisySample =\{ (X_i,\noisy[Y]_i)\}_{i \in [n]}$ where each $(X_i,\noisy[Y]_i)$ is sampled from $\noisyDistribution$ independently. We let $(\noisyDistribution)^{\otimes n}$ denote the corresponding product distribution over samples $\noisySample$, and let $\ExpectationOverNoisySamples$ denote expectation with respect to $(\noisyDistribution)^{\otimes n}$. The sample $\noisySample$ is utilised to train a classifier $\classifier$, which is a random member of $\setOfDecisionRules$, measurable with respect to $\noisySample$. The key difficulty of classification with label noise is that $\noisyDistribution$ and $\cleanDistribution$ may differ. Without further assumptions on the relationship between $\noisyDistribution$ and $\cleanDistribution$ the problem is clearly intractable. We utilise the assumption of \emph{class-conditional label noise} introduced by \cite{scott2013classification}:

\begin{assumption}[Class-conditional label noise]\label{classConditionalNoiseAssumption} We say that $\Prob$ satisfies the class-conditional label noise assumption with parameter $\maxAggregateFlippingProbability \in (0,1)$ if there exists $\probFlipZeroToOne,\probFlipOneToZero \in \left(0,1\right)$ with $\probFlipZeroToOne+\probFlipOneToZero<\maxAggregateFlippingProbability$ such that for Borel sets $A \subset \X$, $\Prob[\noisy[Y]=1|X \in A, \hspace{0mm} Y=0] = \probFlipZeroToOne$ and $\Prob[\noisy[Y]=0|X \in A, \hspace{0mm} Y=1] = \probFlipOneToZero$.
\end{assumption}

The remainder of our assumptions depend solely upon $\cleanDistribution$ and are specified in terms of $\marginalDistribution$ and $\regressionFunction$. We begin with two assumptions which are standard in the literature on non-parametric classification. The first is Tysbakov's margin assumption \cite[]{mammen1999}.

\begin{assumption}[Margin assumption]\label{marginAssumption} Given $\marginExponent \in [0, \infty)$ and $\marginConstant\in [ 1,\infty)$, we shall say that $\Prob$ satisfies the margin assumption with parameters $\left(\marginExponent,\marginConstant\right)$ if the following holds for all $\xi \in (0,1)$, 
\begin{align*}
\marginalDistribution\left( \left\lbrace x \in \X: 0< \left| \regressionFunction(x) -\frac{1}{2}\right| <\xi\right\rbrace \right) \leq \marginConstant \cdot \xi^{\marginExponent}.
\end{align*}
\end{assumption}
We will also assume that the regression function $\regressionFunction$ is H\"{o}lder continuous.

\begin{assumption}[H\"{o}lder assumption]\label{holderAssumption} Given a function $f:\X\rightarrow [0,1]$ and constants $\holderExponent \in (0,1]$, $\holderConstant \geq 1$ we shall say that $f$ satisfies the H\"{o}lder assumption with parameters $(\holderExponent,\holderConstant)$ if for all $x_0, x_1 \in \X$ with $\dist(x_0,x_1)<1$ we have $\left| f(x_0)-f(x_1)\right| \leq \holderConstant \cdot \dist(x_0,x_1)^{\holderExponent}$.
\end{assumption}

We shall also require some assumptions on $\marginalDistribution$. We let $\suppMarginalDistribution \subset \X$ denote the measure-theoretic support of $\marginalDistribution$ and for each $x \in \X$ and $r \in (0,\infty)$ we let $B_r(x):= \left\lbrace z \in \X: \dist(x,z)<r\right\rbrace$.

\begin{assumption}[Minimal mass assumption]\label{minimalMassAssumption} Given $\dimensionExponent  > 0$ and a function $\densityFunction:\X\rightarrow \left(0,1\right)$. We shall say that $\marginalDistribution$ satisfies the minimal mass assumption with parameters $(\dimensionExponent,\densityFunction)$ if we have $\marginalDistribution\left(B_r(x)\right) \geq  \densityFunction(x) \cdot r^d$ for all $x \in \suppMarginalDistribution$ and $r \in \left(0,1\right)$. 
\end{assumption}
\begin{assumption}[Tail assumption]\label{tailAssumption} Given $\tailsExponent \in (0, \infty)$, $\tailsConstant \geq 1$, $\tailsThreshold \in (0,1)$ and a density function $\densityFunction:\X\rightarrow \left(0,1\right)$, we shall say that $\marginalDistribution$ satisfies the tail assumption with parameters $(\tailsExponent, \tailsConstant, \tailsThreshold, \densityFunction)$ if for all $\epsilon  \in (0, \tailsThreshold)$ we have $\marginalDistribution\left(\left\lbrace x \in \X: \densityFunction(x) <\epsilon \right\rbrace\right) \leq \tailsConstant \cdot \epsilon^{\tailsExponent}$.
\end{assumption}

Assumptions \ref{minimalMassAssumption} and \ref{tailAssumption} are natural generalisations to metric spaces of the corresponding assumptions from \cite{gadat2016} in the Euclidean setting. In particular, these assumptions apply to various examples such as Gaussian, Laplace and Cauchy distributions \cite[Table 1]{gadat2016}, with $\densityFunction$ proportional to the probability density function. 

Our final assumption is the most distinctive. It is a quantitative analogue of the mutual irreducibility assumption from \citep{scott2013classification} which implies that $\inf_{x \in \suppMarginalDistribution}\left\lbrace \regressionFunction(x)\right\rbrace = 0$ and\\ $\sup_{x \in \suppMarginalDistribution}\left\lbrace \regressionFunction(x)\right\rbrace = 1$. Rather than assume the existence of positive measure regions of the feature space upon which $\regressionFunction$ is uniformly zero and one, as required for the finite sample rates in \citep[Theorem 2]{scott2015rate}, \cite[Theorem 14]{blanchard2016classification}, we make a weaker assumption that governs the rate at which the regression function approaches its extrema in the tail of the distribution.

\begin{assumption}[Quantitative range assumption]\label{quantitativeRangeAssumption} Given $\fastLimitsExponent\in (0,\infty)$, $\fastLimitsConstant\geq 1$, $\fastLimitsThreshold \in (0,1)$ and a function $\densityFunction:\X\rightarrow \left(0,1\right)$, we shall say that $\Prob$ satisfies the quantitative range assumption with parameters $\left(\fastLimitsExponent, \fastLimitsConstant, \fastLimitsThreshold, \densityFunction\right)$ if for all $\epsilon \in (0,\fastLimitsThreshold)$ we have \[\max\left\lbrace \inf_{x \in \suppMarginalDistribution}\left\lbrace \regressionFunction(x): \densityFunction(x) >\epsilon \right\rbrace, \inf_{x \in \suppMarginalDistribution}\left\lbrace 1-\regressionFunction(x): \densityFunction(x) >\epsilon \right\rbrace\right\rbrace \leq \fastLimitsConstant \cdot \epsilon^{\fastLimitsExponent}.\]
\end{assumption}
If there are regions $S_{\min}\subset \X$ and $S_{\max}\subset \X$ with positive measure $\min\{\marginalDistribution(S_{\min}),\marginalDistribution(S_{\max})\}>0$ such that $\forall x \in S_{\min}$, $\regressionFunction(x)=0$ and $\forall x \in S_{\max}$, $\regressionFunction(x)=1$ then Assumption \ref{quantitativeRangeAssumption} holds with arbitrarily large $\fastLimitsExponent>0$ (see Figure \ref{figForQuantitativeRangeAssumption} (A)). More generally, if there exists $x_{\min}, x_{\max} \in \suppMarginalDistribution$ with $\min\{\densityFunction(x_{\min}),\densityFunction(x_{\min})\}>0$, $\regressionFunction(x_{\min})=0$ and $\regressionFunction(x_{\max})=1$  then Assumption \ref{quantitativeRangeAssumption} again holds with arbitrarily large $\fastLimitsExponent>0$ (see Figure \ref{figForQuantitativeRangeAssumption} (B)). However, Assumption \ref{quantitativeRangeAssumption} can also hold in scenarios in which the extrema of the regression function approaches its extrema gradually in the tails of the distribution. For example, consider a family of distributions $\{\Prob_{\fastLimitsExponent}\}_{\fastLimitsExponent>0}$ where for each $\fastLimitsExponent$, $\Prob_{\fastLimitsExponent}$ has a marginal distribution $\marginalDistribution$ equal to the standard Laplace measure on $\R$ with probability density function $p(x)=\frac{1}{2}\cdot e^{-|x|}$ and regression function $\regressionFunction_{\fastLimitsExponent}(x)=1/(1+e^{-\fastLimitsExponent\cdot x})$ (see Figure \ref{figForQuantitativeRangeAssumption} (C)). For each $\fastLimitsExponent>0$, $\regressionFunction_{\fastLimitsExponent}$ does not attain its extrema, yet Assumption \ref{quantitativeRangeAssumption} holds with exponent $\fastLimitsExponent$. The exponent $\fastLimitsExponent$ controls the rate at which the regression function approaches its extrema as the density function $\densityFunction \approx p$ approaches zero.

\begin{figure}[!htb]
\center{\includegraphics[width=\textwidth]{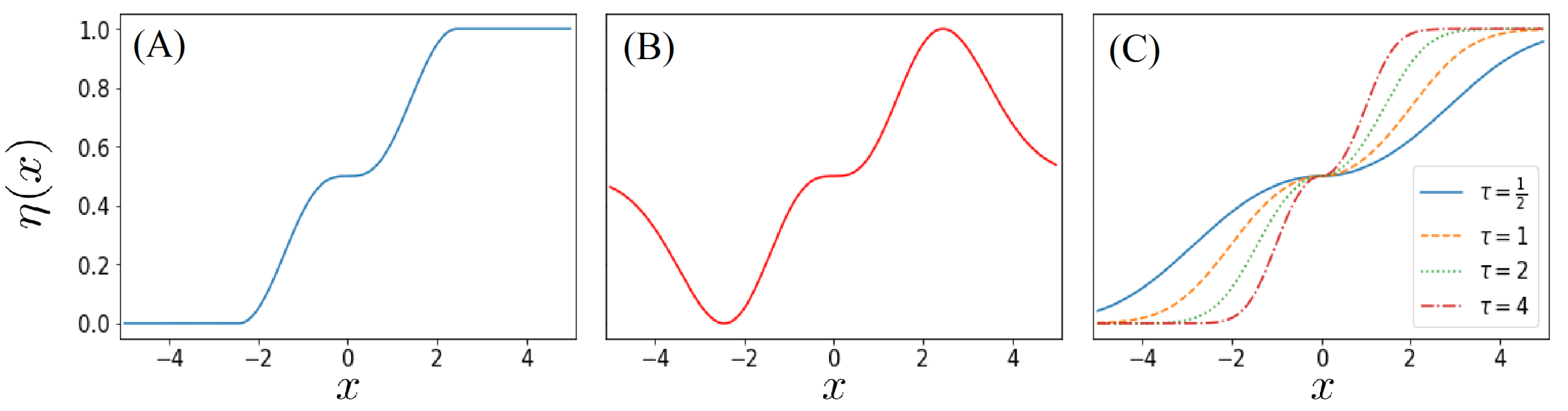}}\vspace{-0.8cm}
\caption{\label{figForQuantitativeRangeAssumption} (A) An example of a regression function in which both the maximum and the minimum of $\regressionFunction$ are attained uniformly on sets of positive measure. (B) An example in which  both the maximum and the minimum of $\regressionFunction$ are each attained at a single point. (C) A family of examples in which the regression function $\regressionFunction(x)=1/(1+e^{-\fastLimitsExponent\cdot x})$ does not attain its extrema. The marginal distribution $\marginalDistribution$ is the standard Laplace measure on $\R$ with density $p(x)=\frac{1}{2}\cdot e^{-|x|}$. In each case, Assumption \ref{quantitativeRangeAssumption} holds with the corresponding exponent $\fastLimitsExponent$.}
\end{figure}

\newpage
In what follows we consider the following class of distributions.
\begin{defn}[Measure classes]\label{measureClassesDefn} Take $\Gamma = (\maxAggregateFlippingProbability,\dimensionExponent,(\marginExponent,\marginConstant),(\holderExponent,\holderConstant), (\tailsExponent,\tailsThreshold,\tailsConstant), ( \fastLimitsExponent, \fastLimitsThreshold, \fastLimitsConstant) )$ consisting of exponents $\marginExponent \in [0,\infty)$, $\holderExponent \in (0,1]$, $\dimensionExponent $, $\tailsExponent$, $\fastLimitsExponent \in (0, \infty)$ and constants $\marginConstant, \holderConstant,\tailsConstant, \fastLimitsConstant \geq 1$ and $\maxAggregateFlippingProbability$, $\tailsThreshold$, $\fastLimitsThreshold \in (0,1)$. We let $\distributionClass\left(\Gamma\right)$ denote the set of all distributions $\Prob$ on triples $(X,Y, \noisy[Y]) \in \X \times \Y^2$, where $(\X,\dist)$ is a metric space and there is some function $\densityFunction:\X \rightarrow (0,1)$ such Assumptions \ref{classConditionalNoiseAssumption}, \ref{marginAssumption}, \ref{holderAssumption}, \ref{minimalMassAssumption}, \ref{tailAssumption}, \ref{quantitativeRangeAssumption} hold with the corresponding parameters.
\end{defn}

Now that we have introduced our assumptions we are ready to state our main results.

\section{Minimax rates for classification with unknown class conditional label noise}\label{mainLBSec}

Our first main result gives a minimax lower bound for classification with unknown class conditional label noise on non-compact domains.

\begin{theorem}\label{mainMinimaxThmLowerBound} 
Take $\Gamma = (\maxAggregateFlippingProbability,\dimensionExponent,(\marginExponent,\marginConstant),(\holderExponent,\holderConstant), (\tailsExponent,\tailsThreshold,\tailsConstant), ( \fastLimitsExponent, \fastLimitsThreshold, \fastLimitsConstant) )$ consisting of exponents $\marginExponent \in [0,\infty)$, $\holderExponent \in (0,1]$, $\dimensionExponent \in [\marginExponent\holderExponent,\infty)$, $\tailsExponent \in (0,1]$, $\fastLimitsExponent \in (0, \infty)$ and constants $\marginConstant \geq 4^{\marginExponent}$, $\holderConstant$, $\tailsConstant$, $\fastLimitsConstant \geq 1$ and $\maxAggregateFlippingProbability \in (0,1)$, $\tailsThreshold \in (0,1/24)$, $\fastLimitsThreshold \in (0,1/3)$. There exists a constant $c(\Gamma)>0$, depending solely upon $\Gamma$, such that for all $n \in \N$
\begin{align*}
\inf_{\classifier}\left\lbrace \sup_{\Prob \in \distributionClass(\Gamma)}\left\lbrace \ExpectationOverNoisySamples\left[ \risk\left(\classifier\right) \right] -  \risk\left(\oracle\right) \right\rbrace\right\rbrace  \geq c(\Gamma) \cdot n^{-\min\left\lbrace { \frac{\tailsExponent\holderExponent(\marginExponent+1)}{\tailsExponent(2\holderExponent+\dimensionExponent)+\marginExponent\holderExponent}},{\frac{\fastLimitsExponent \holderExponent(\marginExponent+1)}{\fastLimitsExponent(2\holderExponent+\dimensionExponent)+\holderExponent}} \right\rbrace }.
\end{align*}
The infimum is taken over all classifiers $\classifier$ which are measurable with respect to $\noisySample$.
\end{theorem}

In Section \ref{mainAlgoSec} we shall introduce a classifier which attains the rates in Theorem \ref{mainMinimaxThmLowerBound} up to logarithmic factors, with high-probability (Theorem \ref{mainUpperBoundThm}). Theorem \ref{mainMinimaxThmLowerBound} displays an interesting threshold behaviour not seen in the label noise free setting. When the exponent $\fastLimitsExponent$ is large ($\fastLimitsExponent \cdot \marginExponent \geq \tailsExponent$) and the regression function $\regressionFunction$ approaches its extrema sufficiently quickly, the exponent matches the label noise free rate. However, when the exponent $\fastLimitsExponent$ is smaller ($\fastLimitsExponent \cdot \marginExponent < \tailsExponent$) and the regression function $\regressionFunction$ approaches its extrema more gradually, the learning behaviour deteriorates accordingly.

The proof of Theorem \ref{mainMinimaxThmLowerBound} reflects this threshold behaviour, and may be split into two claims:
\begin{align}\label{knownNoiseLBClaim}
\inf_{\classifier}\left\lbrace \sup_{\Prob \in \distributionClass(\Gamma)}\left\lbrace
\ExpectationOverNoisySamples\left[ \risk\left(\classifier\right) \right] -  \risk\left(\oracle\right) \right\rbrace\right\rbrace& \geq c_0(\Gamma) \cdot n^{-\frac{\holderExponent \tailsExponent(\marginExponent+1)}{\tailsExponent(2\holderExponent+\dimensionExponent)+\marginExponent\holderExponent} },
\end{align}
\begin{align}\label{unknownNoiseLBClaim}
\inf_{\classifier}\left\lbrace \sup_{\Prob \in \distributionClass(\Gamma)}\left\lbrace\ExpectationOverNoisySamples\left[ \risk\left(\classifier\right) \right] -  \risk\left(\oracle\right) \right\rbrace\right\rbrace &\geq c_1(\Gamma) \cdot n^{-\frac{\fastLimitsExponent \holderExponent(\marginExponent+1)}{\fastLimitsExponent(2\holderExponent+\dimensionExponent)+\holderExponent}}.
\end{align}
The lower bound in (\ref{knownNoiseLBClaim}) corresponds to the difficulty of the pure classification problem, with or without label noise. The exponent is the same as that identified in \cite[Theorem 4.5]{gadat2016}. A full proof of claim (\ref{knownNoiseLBClaim}) is presented in Appendix \ref{LBForUncorruptedDataAppendix} (Proposition \ref{lowerBoundWithKnownNoise}). The proof method is broadly similar to that of \cite{gadat2016}, with two key differences. Firstly, our lower bounds hold for non-integer as well as integer dimension $\dimensionExponent$. Secondly, technical adjustments are required to ensure that Assumption \ref{quantitativeRangeAssumption} is satisfied.

\begin{figure}[!htb]
\center{\includegraphics[width=0.9\textwidth]{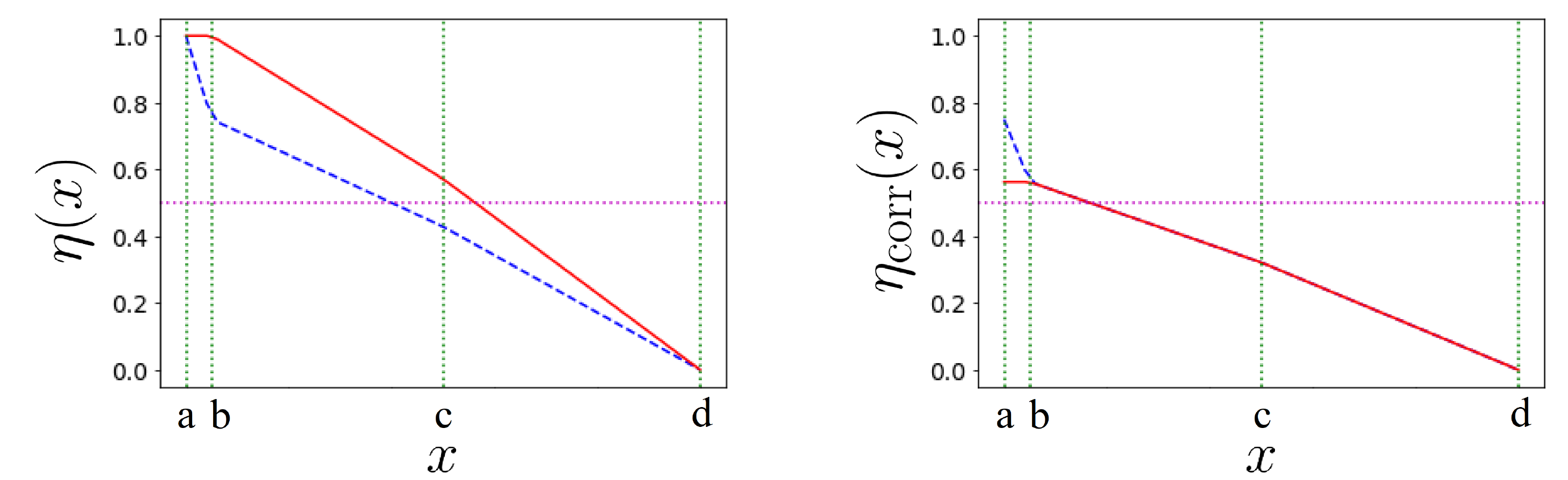}}\vspace{-0.1cm}
\caption{\label{figForLowerBound} An illustration of the construction for the proof of Theorem \ref{mainMinimaxThmLowerBound} when $\fastLimitsExponent \marginExponent \leq \tailsExponent$. A pair of distributions with substantially different regression functions $\regressionFunction$ for which the corresponding corrupted regression functions $\noisy[\regressionFunction]$ are close.
}
\end{figure}

The lower bound in (\ref{unknownNoiseLBClaim}) corresponds to the difficulty of estimating the noise probabilities and the resultant effect upon classification risk. This component of the lower bound is more interesting as it reflects behaviour unseen in the label noise free setting. The proof of the lower bound in (\ref{unknownNoiseLBClaim}) is given in Appendix \ref{LBForUnknownLabelNoiseAppendix} (Proposition \ref{lowerBoundWithUnknownNoise}). The idea is as follows. We construct a pair of distributions such that the corrupted regression functions closely resemble one another, yet the true regression functions are substantially different, and lie on different sides of the classification threshold $\frac{1}{2}$ for large fractions of the feature space. Thus, whilst it is difficult to distinguish the two distributions based upon the corrupted sample $\noisySample$, failing to do so results in substantial increase in risk relative to the Bayes classifier. The construction is illustrated in Figure \ref{figForLowerBound}.

\section{An adaptive algorithm with a minimax optimal upper bound}\label{mainAlgoSec}

In this section we construct a classifier for learning with unknown label noise on non-compact domains. In Section \ref{mainUpperBoundThmSec} we shall present high-probability performance guarantee for the algorithm (Theorem \ref{mainUpperBoundThm}) which matches the minimax lower bound (Theorem \ref{mainMinimaxThmLowerBound}) up to logarithmic factors.

\subsection{Constructing an algorithm for classification with class conditional label noise}\label{algoTemplateSec}

Our methodology is founded on observations due to \cite{menon2015learning}: Define $\noisyRegressionFunction:\X \rightarrow [0,1]$ to be the corrupted regression function, given by $\noisyRegressionFunction(x):=\noisyDistribution[\noisy[Y]=1|X=x]$ for $x \in \X$. By Assumption \ref{classConditionalNoiseAssumption}, $\noisyRegressionFunction$ is related to $\regressionFunction$ by
\begin{align}\label{writingCorruptedRegressionFunctionInTermsOfRegressionFunctionEq}
\noisyRegressionFunction(x)=\left(1-\probFlipZeroToOne-\probFlipOneToZero\right)\cdot \regressionFunction(x)+\probFlipZeroToOne.
\end{align}
Assumption \ref{quantitativeRangeAssumption} implies that $\inf_{x \in \suppMarginalDistribution}\{\regressionFunction(x)\}=0$ and $\sup_{x \in \suppMarginalDistribution}\{\regressionFunction(x)\}=1$. Combining with (\ref{writingCorruptedRegressionFunctionInTermsOfRegressionFunctionEq}) yields
$\inf_{x \in \suppMarginalDistribution}\{\noisyRegressionFunction(x)\}=\probFlipZeroToOne$ and $\sup_{x \in \suppMarginalDistribution}\{\noisyRegressionFunction(x)\}=1-\probFlipOneToZero$. Moreover, relation (\ref{writingCorruptedRegressionFunctionInTermsOfRegressionFunctionEq}) implies
\begin{align*}
\oracle(x) = \one\left\lbrace \regressionFunction(x) \geq {1}/{2}\right\rbrace = \one\left\lbrace \noisyRegressionFunction(x) \geq ({1}/{2})\cdot(1+\probFlipZeroToOne-\probFlipOneToZero)\right\rbrace.
\end{align*}
These observations motivate the `plug-in' style template given in Algorithm \ref{LabelNoiseAlgo}. To instantiate Algorithm \ref{LabelNoiseAlgo} in our setting we require a procedure for estimating the value of the corrupted regression function at a point $\noisy[{\est[\regressionFunction]}](x)$ and a procedure for providing estimates $\hat{M}\left({\noisy[\regressionFunction]}\right)$ and $\hat{M}\left(1-{\noisy[\regressionFunction]}\right)$ for the supremum of $\noisyRegressionFunction$ and $1-\noisyRegressionFunction$, respectively, based on the corrupted sample $\noisySample$. Consequently, we turn to the subject of estimating the values of the corrupted regression function at a point in Section \ref{functionEstimationSec} and to the subject of estimating the extrema of the corrupted regression function in Section \ref{funcMaxSec}. In Section \ref{mainUpperBoundThmSec} we bring these pieces together to provide a concrete instantiation of Algorithm \ref{LabelNoiseAlgo} with a high probability risk bound.

\begin{algorithm}[htbp]

{ 
 \begin{enumerate}[itemsep=2pt,topsep=2pt,parsep=2pt,partopsep=2pt]
     \item Compute an estimate of the corrupted regression function $\noisy[{\est[\regressionFunction]}](x)$ with sample $\noisy[\sample]$;
     \item Estimate $\estProbFlipZeroToOne = 1-\hat{M}\left(1-\noisy[\regressionFunction]\right)$ and $\estProbFlipOneToZero = 1-\hat{M}\left({\noisy[\regressionFunction]}\right)$;
     \item Let $\hat{\decisionRule}(x):= \one\left\lbrace \noisy[{\est[\regressionFunction]}](x) \geq 1/2 \cdot \left(1+\estProbFlipZeroToOne-\estProbFlipOneToZero\right) \right\rbrace$.
     \end{enumerate}
 }
  {\caption{A meta-algorithm for classification with class-conditional label noise.\label{LabelNoiseAlgo}}}
\end{algorithm}
% introduced in Section \ref{algoTemplateSec} and instantiated in Section \ref{mainUpperBoundThmSec}

\subsection{Function estimation with k-nearest neighbours and Lepski's rule}\label{functionEstimationSec}

In this section we consider supervised $k$-nearest neighbour regression. Whilst we are motivated by the estimation of $\noisy[\regressionFunction]$ we shall frame our results in a more general fashion for clarity. Suppose we have an unknown function $f:\X\rightarrow [0,1]$ and a distribution $\Prob_f$ on $\X \times [0,1]$ such that $f(x)= \E\left[Z|X=x\right]$ for all $x \in\X$. In this section we consider the task of to estimating $f$ based on a sample $\sample_f = \left\lbrace \left(X_i,Z_i\right)\right\rbrace_{i \in [n]}$ with $(X_i,Z_i) \sim \Prob_f$ generated i.i.d. Given $x \in \X$ we let $\left\lbrace \tau_{n,q}(x)\right\rbrace_{q \in [n]}$ be an enumeration of $[n]$ such that for each $q \in [n-1]$, $\rho\left(x,X_{\tau_{n,q}(x)}\right) \leq \rho\left(x,X_{\tau_{n,q+1}(x)}\right)$. The $k$-nearest neighbour regression estimator is given by
\begin{align*}
\hat{f}_{n,k}(x):=\frac{1}{k} \cdot \sum_{q \in [k]}Z_{\tau_{n,q}(x)}.
\end{align*}
To apply $\hat{f}_{n,k}$ we must choose a value of $k$. The optimal value of $k$ will depend upon the distributional parameters and the local density $\densityFunction(x)$ at a test point. Inspired by \cite{kpotufe2013adaptivity} we shall use Lepski's method to select $k$. For each $x \in \X$, $n\in \N$, $k \in [n]$ and $\delta \in (0,1)$ we define
\begin{align*}
\hat{\mathcal{I}}_{n,k,\delta}(x):= \left[ \hat{f}_{n,k}(x)- \sqrt{\frac{2\log((4n)/\delta)}{k}},\hat{f}_{n,k}(x)+\sqrt{\frac{2\log((4n)/\delta)}{k}}\right].
\end{align*}
We then let 
\begin{align*}
\hat{k}_{n,\delta}(x):=\max_{k \in \N \cap [8\log(2n/\delta),n/2]}\left\lbrace \bigcap_{q \in \N \cap [8\log(2n/\delta),k]}   \hat{\mathcal{I}}_{n,q,\delta}(x) \neq \emptyset \right\rbrace,
\end{align*}
and define $\hat{f}_{n,\delta}(x):= \hat{f}_{n,k}(x)$ with $k=\hat{k}_{n,\delta}(x)$. Intuitively, the value of $k$ is increased until the bias begins to dominate the variance, which reflects itself in non-overlapping confidence intervals.

\begin{theorem}\label{knnRegressionBoundLepskiK} Suppose that $f$ satisfies the H\"{o}lder assumption with parameters $(\holderExponent, \holderConstant)$ and $\marginalDistribution$ satisfies the minimal mass assumption with parameters $(\dimensionExponent, \densityFunction)$. Given any $n \in \N$, $\delta \in (0,1)$ and $x \in \X$, with probability at least $1-\delta$ over $\sample_f$ we have
\begin{align}\label{statementOfLepskiKNNIneq}
\left| \hat{f}_{n,\delta}(x)-f(x) \right|\leq (8 \sqrt{2}) \cdot   {\holderConstant}^{\frac{\dimensionExponent}{2\holderExponent+\dimensionExponent}} \cdot\left( \frac{\log(4n/\delta)}{\densityFunction(x)\cdot n} \right)^{\frac{\holderExponent}{2\holderExponent+\dimensionExponent}}.
\end{align}
\end{theorem}

A proof of Theorem \ref{knnRegressionBoundLepskiK} is presented in Appendix \ref{proofOfKNNRegressionBoundsAppendix}. The principal difference with \cite{kpotufe2013adaptivity} is that we do not require an upper bound on the $\epsilon$-covering numbers. This is crucial for our setting since the assumption of an upper bound on the $\epsilon$-covering numbers rules out interesting non-compact settings. Theorem \ref{knnRegressionBoundLepskiK} will be applied to the label noise problem in Section \ref{mainUpperBoundThmSec}.

\subsection{A lower confidence bound approach for estimating the supremum of a function}\label{funcMaxSec}

In this section we deal with the problem of estimating the supremum of a function $M(f):=\sup_{x \in \suppMarginalDistribution}\{f(x)\}$. This is motivated by the challenge of estimating the label noise probabilities (Section \ref{algoTemplateSec}). We adopt the general statistical setting from Section \ref{functionEstimationSec}. One might expect to obtain an effective estimator of the maximum by simply taking the empirical maximum of $\hat{f}_{n,\delta}$ over the data. However, this approach is likely to overestimate the maximum in our non-compact setting since estimates at points with low density will have large variance. To mitigate this effect we must subtract a confidence interval. The error due to the variance of $\hat{f}_{n,k}(x)$ can be bounded via Hoeffding's inequality. The error due to bias is more difficult to estimate since it depends upon unknown distributional parameters. Fortunately, for estimating $M(f)$ this is not a problem since the bias at any given point is always negative. This motivates the following simple adaptive estimator: 
\begin{align}\label{estOfFunctionMaxDef}
\hat{M}_{n,\delta}(f):= \max_{(i,k) \in [n]^2}\left\lbrace \hat{f}_{n,k}(X_i)-\sqrt{\frac{\log(4n/\delta)}{k}}\right\rbrace .
\end{align}

\begin{theorem}\label{funcMaxThm} Suppose that $f$ satisfies the H\"{o}lder assumption with parameters $(\holderExponent, \holderConstant)$ and $\marginalDistribution$ satisfies the minimal mass assumption with parameters $(\dimensionExponent, \densityFunction)$. Given any $n \in \N$ and $\delta \in (0,1)$ with probability at least $1-\delta$ over $\sample_f$ we have
\begin{align}\label{ineqsForFuncMaxThm}
\sup_{x \in \suppMarginalDistribution}\left\lbrace f(x) -7 \cdot   {\holderConstant}^{\frac{\dimensionExponent}{2\holderExponent+\dimensionExponent}} \cdot\left( \frac{\log(4n/\delta)}{\densityFunction(x)\cdot n} \right)^{\frac{\holderExponent}{2\holderExponent+\dimensionExponent}}\right\rbrace \leq \hat{M}_{n,\delta}(f) \leq M(f).
\end{align}
\end{theorem}

\begin{proof}
It suffices to show that for any fixed $x_0 \in \X$ with probability at least $1-\delta$ over $\sample_f$ we have
\begin{align}\label{toProveInMaxThmBound}
 f(x_0) -7 \cdot   {\holderConstant}^{\frac{\dimensionExponent}{2\holderExponent+\dimensionExponent}} \cdot\left( \frac{\log(4n/\delta)}{\densityFunction(x_0)\cdot n} \right)^{\frac{\holderExponent}{2\holderExponent+\dimensionExponent}} \leq \hat{M}_{n,\delta}(f) \leq M(f).
\end{align}
Indeed, the bound (\ref{ineqsForFuncMaxThm}) may then be deduced by continuity of measure.\\ Choose $\tilde{k}:= \lfloor \frac{1}{2} \cdot \left(\densityFunction(x_0)\cdot n\right)^{\frac{2\holderExponent}{2\holderExponent+\dimensionExponent}} \cdot \left({ \log(4n/\delta)}\cdot{\holderConstant^{-2}}\right)^{{\dimensionExponent}/{(2\holderExponent+\dimensionExponent)}}\rfloor$. By an application of the multiplicative Chernoff bound (Lemma \ref{closeNeighboursLemma}), the following holds with probability at least $1-\delta/2$,
\begin{align}\label{closeNeighboursAppInMaxThmPf}
\dist\left(x_0,X_{\tau_{n,\tilde{k}}(x)}\right)< \left( \frac{2\tilde{k}}{\densityFunction(x_0) \cdot n} \right)^{\frac{1}{\dimensionExponent}} \leq \xi:=\left(\frac{\log(4n/\delta)}{\holderConstant^2 \cdot \densityFunction(x_0) \cdot n}\right)^{\frac{1}{2\holderExponent+\dimensionExponent}},
\end{align}
provided that $8\log(2n/\delta)\leq \tilde{k}\leq \densityFunction(x_0)\cdot n/2$. By Hoeffding's inequality (see Lemma \ref{knnEstIsCloseToItsConditionalExpectationLemma}) combined with the union bound the following holds simultaneously over all pairs $(i,k) \in [n]^2$ with probability at least $1-\delta/2$,
\begin{align}\label{conditionalExpectationApplicationBoundInMaxThmPf}
\left|  \hat{f}_{n,k}(X_i)- \frac{1}{k}\sum_{q \in [k]}f\left(X_{\tau_{n,q}(X_i)}\right)\right| < \sqrt{\frac{\log(2/(\delta/(2n^2))}{2k}}\leq \sqrt{\frac{\log(4n/\delta)}{k}}.
\end{align}
Let us assume that (\ref{closeNeighboursAppInMaxThmPf}) and (\ref{conditionalExpectationApplicationBoundInMaxThmPf}) hold. By the union bound this is the case with probability at least $1-\delta$. Now take $i_0 =\tau_{n,1}(x_0) \in [n]$.  The upper bound in (\ref{toProveInMaxThmBound}) follows immediately from (\ref{conditionalExpectationApplicationBoundInMaxThmPf}). 

To prove the lower bound in (\ref{toProveInMaxThmBound}) we assume, without loss of generality, that $n$ is sufficiently large that 
$8\log(2n/\delta)\leq \tilde{k}\leq \densityFunction(x_0)\cdot n/2$ and $\tilde{k}\geq   \frac{1}{4} \cdot \left(\densityFunction(x_0)\cdot n\right)^{\frac{2\holderExponent}{2\holderExponent+\dimensionExponent}} \cdot \left({ \log(4n/\delta)}\cdot{\holderConstant^{-2}}\right)^{{\dimensionExponent}/{(2\holderExponent+\dimensionExponent)}}$. Indeed the lower bound is trivial for smaller values of $n$. By (\ref{closeNeighboursAppInMaxThmPf}) combined with the triangle inequality, for each $q \in [\tilde{k}]$ we have $\dist\left(X_{i_0},X_{\tau_{n,q}(x)}\right) \leq \dist\left(x_0,X_{i_0}\right)+\dist\left(x_0,X_{\tau_{n,q}(x)}\right) \leq 2\cdot\xi$,
where $\xi$ is defined in (\ref{closeNeighboursAppInMaxThmPf}). Hence, for each $q \in [\tilde{k}]$ we have $\dist\left(X_{i_0},X_{\tau_{n,q}(X_{i_0})}\right)\leq 2\cdot \xi$. Applying  (\ref{closeNeighboursAppInMaxThmPf}) once again we see that for all $q \in [\tilde{k}]$, we have $\dist\left(x_0,X_{\tau_{n,q}(X_{i_0})}\right) \leq \dist\left(x_0,X_{i_0}\right)+\dist\left(X_{i_0},X_{\tau_{n,q}(X_{i_0})}\right) \leq 3\cdot \xi$. By the H\"{o}lder assumption we deduce that 
\begin{align*}
\left| \frac{1}{k}\sum_{q \in [\tilde{k}]}f\left(X_{\tau_{n,q}(X_{i_0})}\right) -f(x_0) \right| &  \leq \max_{q \in [\tilde{k}]}\left\lbrace \holderConstant\cdot \dist\left(x_0,X_{\tau_{n,q}(X_{i_0})}\right)^{\holderExponent}  \right\rbrace\\
& \leq \holderConstant \cdot \left( 3 \cdot \xi\right)^{\holderExponent} \leq 3 \cdot   {\holderConstant}^{\frac{\dimensionExponent}{2\holderExponent+\dimensionExponent}} \cdot\left( \frac{\log(4n/\delta)}{\densityFunction(x_0)\cdot n} \right)^{\frac{\holderExponent}{2\holderExponent+\dimensionExponent}}.
\end{align*}
Combining with (\ref{conditionalExpectationApplicationBoundInMaxThmPf}) we deduce that
\begin{align*}
\hat{M}_{n,\delta}(f)  &\geq \hat{f}_{n,k}(X_{i_0})-f(x_0)-\sqrt{\frac{\log(4n/\delta)}{\tilde{k}}} \\
&\geq \frac{1}{k}\sum_{q \in [\tilde{k}]}f\left(X_{\tau_{n,q}(X_{i_0})}\right)  -4\cdot \sqrt{\frac{\log(4n/\delta)}{4\tilde{k}}}\\
& \geq f(x_0) - 3 \cdot   {\holderConstant}^{\frac{\dimensionExponent}{2\holderExponent+\dimensionExponent}} \cdot\left( \frac{\log(4n/\delta)}{\densityFunction(x_0)\cdot n} \right)^{\frac{\holderExponent}{2\holderExponent+\dimensionExponent}} -4\cdot \sqrt{\frac{\log(4n/\delta)}{4\tilde{k}}}\\
&\geq f(x_0) - 7 \cdot   {\holderConstant}^{\frac{\dimensionExponent}{2\holderExponent+\dimensionExponent}} \cdot\left( \frac{\log(4n/\delta)}{\densityFunction(x_0)\cdot n} \right)^{\frac{\holderExponent}{2\holderExponent+\dimensionExponent}}.
\end{align*}
This gives the lower bound in (\ref{toProveInMaxThmBound}) and completes the proof of Theorem \ref{funcMaxThm}.
\end{proof}

Theorem \ref{funcMaxThm} implies the following corollary.

\begin{corollary}\label{fastLimitsMaxThmCorollary} Suppose that $f$ satisfies the H\"{o}lder assumption with parameters $(\holderExponent, \holderConstant)$ and $\marginalDistribution$ satisfies the minimal mass assumption with parameters $(\dimensionExponent, \densityFunction)$. Suppose further that for some $\fastLimitsExponent \in (0,\infty]$, $\fastLimitsConstant \geq 1$ and $\fastLimitsThreshold \in (0,1)$, for each $\epsilon \in (0,\fastLimitsThreshold)$ we have $\sup_{x\in \suppMarginalDistribution}\left\lbrace f(x): \densityFunction(x)>\epsilon \right\rbrace \geq M(f)-\fastLimitsConstant \cdot \epsilon^{\fastLimitsExponent}$. Then, for each $n \in \N$ and $\delta \in (0,1)$ with probability at least $1-\delta$ over $\sample_f$,
\begin{align*}
\left|  \hat{M}_{n,\delta}(f) - M(f)\right| \leq 8 \cdot\left(\holderConstant^{\dimensionExponent/\holderExponent}\cdot (\fastLimitsConstant/\fastLimitsThreshold) \right)^{\frac{\holderExponent}{2\holderExponent+\dimensionExponent}}\cdot\left( \frac{\log(4n/\delta)}{ n} \right)^{\frac{\fastLimitsExponent \holderExponent}{\fastLimitsExponent(2\holderExponent+\dimensionExponent)+\holderExponent}}.
\end{align*}
\end{corollary}

\begin{proof}
Combine Theorem \ref{funcMaxThm} with $\sup_{x\in \suppMarginalDistribution}\left\lbrace f(x): \densityFunction(x)>\epsilon \right\rbrace \geq M(f)-\fastLimitsConstant \cdot \epsilon^{\fastLimitsExponent}$ and \[\epsilon = \min \left\lbrace \fastLimitsThreshold, \left( \holderConstant^{\dimensionExponent}/\fastLimitsConstant^{2\holderExponent+\dimensionExponent}\right)^{\frac{ 1}{\fastLimitsExponent(2\holderExponent+\dimensionExponent)+\holderExponent}} \cdot \left( \frac{\log(4n/\delta)}{ n} \right)^{\frac{ \holderExponent}{\fastLimitsExponent(2\holderExponent+\dimensionExponent)+\holderExponent}}\right\rbrace.\]
\end{proof}

Corollary \ref{fastLimitsMaxThmCorollary} highlights the dependency of the maximum estimation method upon the rate at which the function approaches its maximum in the tails of the distribution.

\subsection{A high-probability upper bound for classification with class conditional label noise}\label{mainUpperBoundThmSec}

We now combine the procedures introduced in Sections \ref{functionEstimationSec} and \ref{funcMaxSec} to instantiate the template given in Algorithm \ref{LabelNoiseAlgo}. Given a corrupted sample $\noisySample$ and a confidence parameter $\delta \in (0,1)$ proceed as follows: First, we estimate $\noisyRegressionFunction(x)$ using the $k$-NN method introduced in Section \ref{functionEstimationSec} $\noisy[{\est[\regressionFunction]}](x) =\widehat{\left(\noisy[\regressionFunction]\right)}_{n,\delta^2/3}(x)$. Second, we apply the maximum estimation procedure introduced in Section \ref{funcMaxSec} to obtain estimates $\estProbFlipZeroToOne = 1-\hat{M}_{{n},{\delta/3}}\left(1-\noisy[\regressionFunction]\right)$ and $\estProbFlipOneToZero = 1-\hat{M}_{{n},{\delta/3}}\left({\noisy[\regressionFunction]}\right)$. Third, we take $\hat{\decisionRule}_{n,\delta}(x):= \one\left\lbrace \noisy[{\est[\regressionFunction]}](x) \geq 1/2 \cdot \left(1+\estProbFlipZeroToOne-\estProbFlipOneToZero\right) \right\rbrace$. The classifier $\hat{\decisionRule}_{n,\delta}$ satisfies the high probability risk bound given in Theorem \ref{mainUpperBoundThm}.

\newcounter{Kcounter}
\setcounter{Kcounter}{-1}

\newcommand{\KConstant}{\addtocounter{Kcounter}{1}K_{\arabic{Kcounter}}}
\newcommand{\goodDeltaSet}{\mathcal{G}_{\delta}}
\newcommand{\epsDeltaN}{\epsilon\left(n,\delta\right)}

\begin{theorem}\label{mainUpperBoundThm} Take $\Gamma = (\maxAggregateFlippingProbability,\dimensionExponent,(\marginExponent,\marginConstant),(\holderExponent,\holderConstant), (\tailsExponent,\tailsThreshold,\tailsConstant), ( \fastLimitsExponent, \fastLimitsThreshold, \fastLimitsConstant) )$ consisting of exponents $\marginExponent \in [0,\infty)$, $\holderExponent \in (0,1]$, $\dimensionExponent \in (0,\infty)$, $\tailsExponent \in (\holderExponent/(2\holderExponent+\dimensionExponent),\infty)$, $\fastLimitsExponent \in (0, \infty)$ and constants $\maxAggregateFlippingProbability \in (0,1)$, $\marginConstant$, $ \holderConstant, \tailsConstant, \fastLimitsConstant \geq 1$ and $\tailsThreshold$, $\fastLimitsThreshold \in (0,1)$. Then there exists a constant $C(\Gamma)$ depending solely upon $\Gamma$ such that for any $n \in \N$ and $\delta \in (0,1)$ the following risk bound holds with probability at least $1-\delta$ over the corrupted data sample $\noisySample$,
\begin{align*}\risk\left(\est[\decisionRule]_{n,\delta} \right) - \risk\left( \oracle\right) \leq C(\Gamma)  \cdot \left(\frac{\log(n/\delta)}{n}\right)^{\min\left\lbrace { \frac{\tailsExponent\holderExponent(\marginExponent+1)}{\tailsExponent(2\holderExponent+\dimensionExponent)+\marginExponent\holderExponent}},{\frac{\fastLimitsExponent \holderExponent(\marginExponent+1)}{\fastLimitsExponent(2\holderExponent+\dimensionExponent)+\holderExponent}} \right\rbrace } + \delta.
\end{align*}

\end{theorem}
A full proof of Theorem \ref{mainUpperBoundThm} is presented in Appendix \ref{proofOfHPUpperBoundAppendix}. By Theorem \ref{mainMinimaxThmLowerBound} the classifier $\hat{\decisionRule}_{n,\delta}$ is minimax optimal up to logarithmic factor. We emphasise that the classifier $\hat{\decisionRule}_{n,\delta}$ is fully \emph{adaptive} and does not require any prior knowledge of either the local density $\densityFunction(x)$, or the distributional parameters.

\begin{comment}
Old formats

\begin{align*}\risk\left(\est[\decisionRule]_{n,\delta} \right) - \risk\left( \oracle\right) \leq \begin{cases}
{C}\cdot{\left(1-\probFlipZeroToOne-\probFlipOneToZero\right)^{-2(1+\marginExponent)}}  \cdot \left(\frac{\log(n/\delta)}{n}\right)^{ \frac{\tailsExponent\holderExponent(\marginExponent+1)}{\tailsExponent(2\holderExponent+\dimensionExponent)+\marginExponent\holderExponent}} + \delta &\text{ when }\fastLimitsExponent \marginExponent\geq \tailsExponent\\
{C}\cdot {\left(1-\probFlipZeroToOne-\probFlipOneToZero\right)^{-2(1+\marginExponent)}}  \cdot \left(\frac{\log(n/\delta)}{n}\right)^{\frac{\fastLimitsExponent \holderExponent(\marginExponent+1)}{\fastLimitsExponent(2\holderExponent+\dimensionExponent)+\holderExponent}} + \delta  &\text{ otherwise. }
\end{cases}
\end{align*}
\begin{align*}\risk\left(\est[\decisionRule]_{n,\delta} \right) - \risk\left( \oracle\right) \leq 
{C}\cdot{\left(1-\probFlipZeroToOne-\probFlipOneToZero\right)^{-2(1+\marginExponent)}}  \cdot \left( \left(\frac{\log(n/\delta)}{n}\right)^{ \frac{\tailsExponent\holderExponent(\marginExponent+1)}{\tailsExponent(2\holderExponent+\dimensionExponent)+\marginExponent\holderExponent}}+\left(\frac{\log(n/\delta)}{n}\right)^{\frac{\fastLimitsExponent \holderExponent(\marginExponent+1)}{\fastLimitsExponent(2\holderExponent+\dimensionExponent)+\holderExponent}}\right) + \delta.
\end{align*}

\end{comment}

\section{Related work}\label{relatedWorkSec}

\paragraph{Classification with label noise} The problem of learning a classifier from data with corrupted labels has been widely studied (\cite{Frenay1}). Broadly speaking, there are two approaches to addressing this problem from a theoretical perspective. One approach is to assume that the label noise is either symmetric (but possibly instance dependent) or becomes symmetric as the regression function approaches $1/2$. In this setting the optimal decision boundary does not differ between test and train distributions and classical approaches such as  $k$-nearest neighbours are consistent with finite sample rates \cite[]{cannings2018,MenonMLJ2018}. In turn, our focus is on class-conditional label noise for which the optimal decision boundary will typically differ between test and train distributions and classical algorithms will no longer be consistent. \cite{natarajan2013learning} demonstrated that classification with class-conditional label noise is reducible to classification with a shifted threshold, provided that the noise probabilities are known. This method has been generalised to provide empirical risk minimisation based approaches for various objectives when one only has access to corrupted data \cite[]{natarajan2018,van2018theory}. \cite{scott2015rate} demonstrated that the label noise probabilities may be estimated from the corrupted sample at a rate of $O(1/\sqrt{n})$ provided that there exists a family of sets $\mathcal{S}$ of finite VC dimension with $S_0$, $S_1 \in \mathcal{S}$ such that $\min\{\marginalDistribution(S_0),\marginalDistribution(S_1)\}>0$, $\forall x \in S_0$, $\regressionFunction(x)=0$ and $\forall x_1 \in S_1$, $\regressionFunction(x_1)=1$. This gives rise to a finite sample rate of $O(1/\sqrt{n})$ for classification with unknown label noise over hypothesis classes of bounded VC dimension \citep{scott2015rate,blanchard2016classification}. \cite{ramaswamy2016mixture} has provided an alternative approach to estimating label noise probabilities at a rate of $O(1/\sqrt{n})$. However, the bound requires a separability condition in a Hilbert space, which does not apply in our setting.  \cite{gao2018} gave an adaptation of the $k$-nearest neighbour ($k$-NN) method and prove convergence to the Bayes risk. \cite{reeveKaban19a} obtained minimax optimal fast rates for the method of \cite{gao2018} under the measure smoothness assumptions of \cite{chaudhuri2014rates,doring2018} combined with the mutual irreducibility condition. In both \citep{scott2015rate,blanchard2016classification} and \citep{reeveKaban19a} the assumptions ensure that the regression function is close to its extrema on sets of large measure. This implies that the statistical difficulty of estimating the label noise probabilities is dominated by the difficulty of the classification problem. Consequently, in both cases, the finite sample rates for classification with unknown label noise match the optimal rates for the corresponding label noise free setting, up to logarithmic terms \citep{blanchard2016classification,reeveKaban19a}. In this work we have studied a non-compact setting which includes examples where the minimax optimal rates for learning with label noise are strictly greater than those for learning without label noise.

\paragraph{Non-parametric classification in unbounded domains} The problem of non-parametric classification on non-compact domains where the marginal density is not bounded from below has received some recent attention. One approach is the measure-theoretic smoothness assumption of \cite[]{chaudhuri2014rates,doring2018} whereby deviations in the regression function are assumed to scale with the measure of metric balls. This means that the regression function must become increasingly smooth (i.e. smaller Lipschitz constant) as the density approaches zero. In this work we have adopted the less restrictive approach of \cite{gadat2016} where the Lipschitz constant is not controlled by the density. Instead assumptions are made which bound the measure of the tail of the distribution (Assumption \ref{tailAssumption}). This more flexible setting includes natural examples \cite[Table 1]{gadat2016} and results in optimal convergence rates which are provably slower than those achieved with densities bounded from below. The primary difference between our setting and that of \cite{gadat2016} is that we allow for class-conditional label noise with unknown label noise probabilities. This requires alternative techniques and can result in different optimal rates (Theorems \ref{mainMinimaxThmLowerBound} and \ref{mainUpperBoundThm}). In addition, our method is adaptive to the unknown distributional parameters and local density, unlike the local $k$-NN method of \cite{gadat2016} which assumes prior knowledge of the local density at a test point. This adaptivity is especially significant in the label noise setting where one cannot tune hyper-parameters by minimising the classification error on a hold out set. In order to tune $k$ we use the Lepski method  \cite[]{lepski1997optimal}. Our use of the Lepski method is drawn from the work of \cite{kpotufe2013adaptivity} who applied this method to kernel regression. The principal difference is that whereas \cite{kpotufe2013adaptivity} establish a uniform bound which holds simultaneously for all test points, we only require a pointwise bound. The major advantage of this is that we are able to avoid the restrictive assumption of an upper bound on the $\epsilon$-covering numbers (which would rule out non-compact domains of interest). An alternative approach to non-compact domains has been pursued by \cite{cannings2017local}. Whilst we follow \cite{gadat2016} in bounding the measure of the regions of the feature space where the density falls below a given value (see Assumption \ref{tailAssumption}), \cite{cannings2017local} instead employ a moment assumption. Note that whereas \cite{cannings2017local} make use of an additional set of unlabelled data to locally tune the optimal value of $k$, our method is optimally adaptive without any additional data.

\begin{comment}

%Finally, our results hold for general metric spaces and apply to natural classes of distributions supported on manifolds \cite[Section 4]{jiang2017density}. 

\paragraph{Non-parametric regression} We shall leverage ideas from non-parametric regression to estimate the corrupted regression function i.e. the conditional probability that the corrupted label is a one. We opt for the conceptually simple $k$-NN method. \cite{kpotufe2011k} has demonstrated that the classical $k$-NN method adapts to the intrinsic dimensionality of the data in metric spaces. In order to tune $k$ we use the Lepski method  \cite[]{lepski1997optimal}. This allows to choose $k$ locally in a way which adapts to the unknown H\"{o}lder exponent, dimension and local density. 

\end{comment}

\paragraph{Supremum estimation} Central to our method is the observation of \cite{menon2015learning} that under the mutual irreducibility assumption the noise probabilities may be determined by estimating the extrema of the corrupted regression function. This leads to the problem of determining the supremum of a function on an unbounded metric space based on labelled data. This is closely related to the problem of mode estimation studied by \cite{dasgupta2014optimal} in an unsupervised setting and by \cite{Jiang18} in a supervised setting. The primary difference is that whereas we are only interested in estimating the value of the supremum, those papers focus on estimating the point in the feature space which attains the supremum. This is a more challenging problem which requires strong assumptions including a twice differentiable function. In our setting the feature space is not assumed to have a differentiable structure, so such assumptions cannot be applied. Note also that the sup norm bound of \cite{Jiang18} does not hold in our setting since it requires a uniform lower bound on the density. Our problem is also related to the simple regret minimisation problem in $\X$-armed bandits \cite[]{bubeck2011x,locatelli2018adaptivity} in which the learner actively selects points in the feature space in order to locate and determine the supremum. However, the techniques are quite different, owing to the active rather than passive nature of the problem. In particular, there is no marginal distribution over the feature vectors, since these are selected by the learner. In our setting, conversely, the behaviour of the marginal distribution plays an absolutely crucial role.

\section{Conclusion}

We have determined the minimax optimal learning rate (up to logarithmic factors) for classification in the presence of unknown class-conditional label noise on non-compact metric spaces. The rate displayed an interesting threshold behaviour depending upon the rate at which the regression function approaches its extrema in the tails of the distribution. In addition, we presented an adaptive classification algorithm that attains the minimax rates without prior knowledge of the distributional parameters or the local density.

\acks{This work is funded by EPSRC under Fellowship grant EP/P004245/1. The authors would like to thank Nikolaos Nikolaou and Timothy I. Cannings for useful discussions. We would also like to thank the anonymous reviewers for their careful feedback which led to several improvements in the presentation.}

\newpage
\appendix

\section{Proof of the lower bound}\label{ProofOfLowerBoundAppendix}

In this section we shall present the proof of the main lower bound - Theorem \ref{mainMinimaxThmLowerBound}. The proof of Theorem \ref{mainMinimaxThmLowerBound} consists of two components. The first component corresponds to the difficulty of estimating the noise probabilities and the resultant effect upon classification risk. This component is presented in Proposition \ref{lowerBoundWithUnknownNoise} in Section \ref{LBForUnknownLabelNoiseAppendix}. The second component corresponds to the difficulty of the core classification problem which would have been present even if the learner had access to clean labels. This component is presented in Proposition \ref{lowerBoundWithKnownNoise} in Section \ref{LBForUncorruptedDataAppendix}. Theorem \ref{mainMinimaxThmLowerBound} follows immediately from Propositions \ref{lowerBoundWithUnknownNoise} and \ref{lowerBoundWithKnownNoise}.

Before presenting Propositions \ref{lowerBoundWithUnknownNoise} and \ref{lowerBoundWithKnownNoise} we shall remind the reader of some notation that will be useful in the proof of the lower bound. Recall that we have a distribution $\Prob$ over triples $(X,Y,\noisy[Y])$. We let $\cleanDistribution$ denote the marginal distribution over $(X,Y)$ and $\noisyDistribution$ denote the marginal distribution over $(X,\noisy[Y])$. In addition, we let $\ProbOverCleanSamples$ denote the product distribution over clean samples $\sample=\{(X_i,Y_i)\}_{i \in [n]}$ with $(X_i,Y_i)$ sampled from $\cleanDistribution$ independently and let $\ProbOverNoisySamples$ denote the product distribution over corrupted samples $\noisySample=\{(X_i,\noisy[Y]_i)\}_{i \in [n]}$ with $(X_i,\noisy[Y]_i)$ sampled from $\noisyDistribution$. Similarly, we let $\ExpectationOverCleanSamples$ denote the expectation over clean samples $\sample \sim \ProbOverCleanSamples$ and let $\ExpectationOverNoisySamples$ denote the expectation over corrupted samples $\noisySample \sim \ProbOverNoisySamples$.

\subsection{A lower bound for unknown label noise}\label{LBForUnknownLabelNoiseAppendix}

The goal of this section is to prove Proposition \ref{lowerBoundWithUnknownNoise} which corresponds to the difficulty of estimating the noise probabilities and the resultant effect upon classification risk.

\begin{prop}\label{lowerBoundWithUnknownNoise} 
Take $\Gamma = (\maxAggregateFlippingProbability,\dimensionExponent,(\marginExponent,\marginConstant),(\holderExponent,\holderConstant), (\tailsExponent,\tailsThreshold,\tailsConstant), ( \fastLimitsExponent, \fastLimitsThreshold, \fastLimitsConstant) )$ consisting of exponents $\marginExponent \in [0,\infty)$, $\holderExponent \in (0,1]$, $\dimensionExponent \in [\marginExponent\holderExponent,\infty)$, $\tailsExponent \in (0,1]$, $\fastLimitsExponent \in (0, \infty)$ and constants $\maxAggregateFlippingProbability \in (0,1)$, $\marginConstant \geq 4^{\marginExponent}$, $\marginConstant$, $\holderConstant$, $\tailsConstant$, $\fastLimitsConstant \geq 1$ and $\tailsThreshold \in (0,1/24)$, $\fastLimitsThreshold \in (0,1/3)$. There exists a constant $c_0(\Gamma)$, depending solely upon $\Gamma$, such that for any $n \in \N$ and any classifier $\classifier$ which is measurable with respect to the corrupted sample $\noisySample$, there exists a distribution $\Prob \in \distributionClass(\Gamma)$ such that 
\begin{align*}
\ExpectationOverNoisySamples\left[ \risk\left(\classifier\right) \right] -  \risk\left(\oracle\right) \geq c_1(\Gamma) \cdot n^{-\frac{\fastLimitsExponent \holderExponent(\marginExponent+1)}{\fastLimitsExponent(2\holderExponent+\dimensionExponent)+\holderExponent}}.
\end{align*}
\end{prop}

\newcommand{\volA}{u}
\newcommand{\densityA}{w}
\newcommand{\volC}{v}

To prove Proposition \ref{lowerBoundWithUnknownNoise} we will show that there exists a pair of distributions $\Prob^0$ and $\Prob^1$ such that whilst the corrupted regression functions ($\noisy[\regressionFunction]^0$ and $\noisy[\regressionFunction]^1$) closely resemble one another, the true regression functions ($\regressionFunction^0$ and $\regressionFunction^1$) are substantially different. Thus, whilst it is difficult to distinguish $\Prob^0$ and $\Prob^1$ based upon the corrupted sample $\noisySample$, failing to do so results in substantial misclassification error. Figure \ref{figForLowerBound} in Section \ref{mainLBSec} illustrates the construction. To formalise this idea we require the following variant of Fano's lemma due to \cite{birge2001new}.

\begin{lem}[Birg\'{e}]\label{birgesVersionOfFanosLemma} Given a finite family $\mathcal{S}$ consisting of probability measures on a measureable space $\left(\mathcal{Z},\Sigma\right)$ and a random variable $Z$ with an unknown distribution in the family, then we have
\begin{align*}
\inf_{\hat{T}}\left\lbrace \sup_{{\Prob}_Z\in \mathcal{S}}\left\lbrace \Prob_Z\left[\hat{T}(Z) \neq \Prob_Z\right]\right\rbrace\right\rbrace \geq \min\left\lbrace 0.36,1- \inf_{\Prob_Z \in \mathcal{S}}\left\lbrace \sum_{\tilde{\Prob}_Z \in \mathcal{S}} \frac{\kullbackLeiblerDivergence\left(\Prob_Z,\tilde{\Prob}_Z\right)}{\left|\mathcal{S}\right|\log\left|\mathcal{S}\right|}  \right\rbrace\right\rbrace,
\end{align*}
where the infimum is taken over all measureable (possibly randomised) estimators $\hat{T}:\mathcal{Z}\rightarrow \mathcal{S}$.
\end{lem}

\newcommand{\twoMeasuresUnknownNoiseDistIndex}{\iota}

We apply Lemma \ref{birgesVersionOfFanosLemma} as follows: Given an integer $n \in \N$ and a quintuple $\left(\Delta,r, \volA, \volC,\densityA \right) \in (0,1/6)^5$ (to be selected in terms of $n$ later) we shall construct a measureable space with a pair of distributions. First we construct a metric space by letting $\X = \{a,b,c,d\}$ and choosing $\dist$ such that 
\begin{align*}
\dist(x_0,x_1)& =  r &\text{ if }&x_0,x_1 \in \{a,b\}\text{ and }x_0\neq x_1\\
\dist(x_0,x_1)& \geq 1 &\text{ if }&x_0 \notin \{a,b\}\text{ or }x_1 \notin\{a,b\} \text{ and }x_0\neq x_1\\
\dist(x_0,x_1)&=0 &\text{ if }&x_0 = x_1.
\end{align*}
Note that there are metric spaces $(\X,\dist)$ of this form embedded isometrically in any Euclidean space $(\R^D,\|\cdot\|_2)$. We shall define a pair of probability distributions $\Prob^0,\Prob^1$ over random triples $(X,Y,\tilde{Y}) \in \X\times \Y^2$ as follows. First we define a Borel probability measure $\marginalDistribution$ on $\X$ by $\marginalDistribution(\{a\}) = \volA$, $\marginalDistribution(\{b\}) = 1/3$, $\marginalDistribution(\{c\}) = \volC$ and $\marginalDistribution(\{d\}) = 2/3-\volA-\volC$. Second, we define a pair of regression functions $\regressionFunction^0,\regressionFunction^1:\X\rightarrow [0,1]$ on $\X$ as follows by 
\begin{align*}
\regressionFunction^0(a)&=1,&
\hspace{1cm}\regressionFunction^0(b)&=1-\Delta,& \hspace{1cm}\regressionFunction^0(c)&=\frac{1-\Delta}{2-\Delta}&
\hspace{1cm}\regressionFunction^0(d)&=0\\
\regressionFunction^1(a)&=1,&
\regressionFunction^1(b)&=1,&
\regressionFunction^1(c)&=\frac{1}{2-\Delta}&
\regressionFunction^1(d)&=0.
\end{align*}
Third, we define probabilities $\probFlipFrom[j]^{\twoMeasuresUnknownNoiseDistIndex}\in (0,1)$ for $\{\twoMeasuresUnknownNoiseDistIndex,j\} \in \{0,1\}$ by taking $\probFlipZeroToOne^0=\probFlipOneToZero^0=0$, $\probFlipOneToZero^0= \maxAggregateFlippingProbability/4$ and $\probFlipOneToZero^1=\Delta+(\maxAggregateFlippingProbability/4)\cdot (1-\Delta)$. We then put these pieces together by taking, for each ${\twoMeasuresUnknownNoiseDistIndex}\in \{0,1\}$, $\Prob^{\twoMeasuresUnknownNoiseDistIndex}$ to be the unique distribution on $(X,Y,\tilde{Y}) \in \X\times \Y^2$ with (a) marginal distribution $\marginalDistribution$, (b) regression function $\regressionFunction^{\twoMeasuresUnknownNoiseDistIndex}(x)=\Prob^{\twoMeasuresUnknownNoiseDistIndex}[Y=1|X=x]$ and (c) label noise probabilities $\probFlipFrom[j]^{\twoMeasuresUnknownNoiseDistIndex}$. In addition, we define $\densityFunction:\X \rightarrow (0,1)$ by $\densityFunction(a)=\densityA$, $\densityFunction(c) = \volC$ and $\densityFunction(b) =\densityFunction(d)=1/3$.

\begin{restatable}{lem}{propertiesCheckLemmaForUnknownNoiseLB}\label{propertiesCheckLemmaForUnknownNoiseLB} For $\twoMeasuresUnknownNoiseDistIndex \in \{0,1\}$ the measures $\Prob^{\twoMeasuresUnknownNoiseDistIndex}$ satisfy the following properties:
\begin{enumerate}[label=\alph*)]
    \item $\Prob^{\twoMeasuresUnknownNoiseDistIndex}$ satisfies Assumption \ref{classConditionalNoiseAssumption} with parameter $\maxAggregateFlippingProbability$ provided $\Delta \leq \maxAggregateFlippingProbability/2$;
    \item $\Prob^{\twoMeasuresUnknownNoiseDistIndex}$ satisfies Assumption \ref{marginAssumption} with parameters $(\marginExponent, \marginConstant)$ whenever $\marginConstant\geq 4^{\marginExponent}$ \& $\volC \leq \Delta^{\marginExponent}$;
    \item $\regressionFunction^{\twoMeasuresUnknownNoiseDistIndex}$ satisfies Assumption \ref{holderAssumption} with parameters $(\holderExponent, \holderConstant)$ whenever $\Delta \leq \holderConstant \cdot r^{\holderExponent}$;
    \item $\marginalDistribution$ satisfies Assumption \ref{minimalMassAssumption} with parameters $(\dimensionExponent, \densityFunction)$ whenever $\volA \geq \densityA\cdot r^{\dimensionExponent}$;
    \item $\marginalDistribution$ satisfies Assumption \ref{tailAssumption} with parameters $(\tailsExponent, \tailsConstant, \tailsThreshold, \densityFunction)$ whenever $\tailsExponent \leq 1$, $\tailsThreshold \leq \frac{1}{3}$ \& $\volA \leq \densityA$;
    \item $\Prob^{\twoMeasuresUnknownNoiseDistIndex}$ satisfies Assumption \ref{quantitativeRangeAssumption} with parameters $(\fastLimitsExponent, \fastLimitsConstant,\fastLimitsThreshold,\densityFunction)$ whenever $\fastLimitsThreshold \leq 1/3$ \& $\Delta \leq \fastLimitsConstant \cdot \densityA^{\fastLimitsExponent}$. 
\end{enumerate}
\end{restatable}

\begin{proof} We check each property in turn.

\noindent Property \ref{classConditionalNoiseAssumption} follows immediately from the construction of $\Prob^{\twoMeasuresUnknownNoiseDistIndex}$  and the definitions of $\probFlipFrom[j]^{\twoMeasuresUnknownNoiseDistIndex}$.

\noindent Property \ref{marginAssumption} follows from the fact that since $\Delta<1/6$, we have $\left|\regressionFunction^{\twoMeasuresUnknownNoiseDistIndex}(x)-1/2\right| \geq 1/3$ for $x \neq c$ and 
$\left|\regressionFunction^{\twoMeasuresUnknownNoiseDistIndex}(c)-1/2\right|\geq \Delta/4$. Property \ref{holderAssumption} follows from the fact that the only two distinct points $x_0, x_1$ with $\dist(x_0,x_1)<1$ are $a$ \& $b$ with $\dist(a,b)=r$ and $\left|\regressionFunction^{\twoMeasuresUnknownNoiseDistIndex}(a)-\regressionFunction^{\twoMeasuresUnknownNoiseDistIndex}(b)\right|\leq \Delta$.

\noindent Property \ref{minimalMassAssumption} follows from the fact that  $\marginalDistribution$ is defined by $\marginalDistribution(\{a\}) = \volA$, $\marginalDistribution(\{b\}) = 1/3$, $\marginalDistribution(\{c\}) = \volC$, $\marginalDistribution(\{d\}) = 2/3-\volA-\volC$ and $\densityFunction:\X \rightarrow (0,1)$ by $\densityFunction(a)=\densityA$, $\densityFunction(c) = \volC$ and $\densityFunction(b) =\densityFunction(d)=1/3$. In particular, for $x \neq a$ we have for $\tilde{r} \in (0,1)$ \[\marginalDistribution(B_{\tilde{r}}(x))\geq \marginalDistribution(\{x\}) \geq \densityFunction(x) \geq  \densityFunction(x) \cdot {\tilde{r}}^d.\]
On the other hand, for $x = a$ there are two cases. If ${\tilde{r}} \in (r,1)$ then \[\marginalDistribution(B_{\tilde{r}}(a))\geq \marginalDistribution(\{b\}) \geq 1/3 \geq \densityA \cdot {\tilde{r}}^{\dimensionExponent}=\densityFunction(a)\cdot {\tilde{r}}^{\dimensionExponent} \] since $\densityA <1/6$. If ${\tilde{r}}\leq r$ then $\marginalDistribution(B_{\tilde{r}}(a)) \geq \marginalDistribution(\{a\}) = \volA \geq \densityA \cdot r^d \geq \densityFunction(a)\cdot {\tilde{r}}^{\dimensionExponent}$.

\noindent Property \ref{tailAssumption} requires three cases. If $\epsilon \in [\max\{\densityA, \volC\},1/3)$ then we have
\[\marginalDistribution\left(\{x \in \X: \densityFunction(x)<\epsilon\}\right) =\marginalDistribution\left(\{a,c\}\right) = \volA+\volC \leq 2\max\{\densityA,\volC\} \leq \tailsConstant \cdot \epsilon^{\tailsExponent}.\]
If $\epsilon \in [\min\{\densityA, \volC\},\max\{\densityA, \volC\})$ then we take $x_0 \in \{a,c\}$ with $\densityFunction(x_0) = \min_{x \in \X}\{\densityFunction(x)\}$. Since $\volA \leq \densityA$ we have 
\[\marginalDistribution\left(\{x \in \X: \densityFunction(x)<\epsilon\}\right) =\marginalDistribution\left(\{x_0\}\right) \leq \densityFunction(x_0) = \min\{\densityA, \volC\}\leq \epsilon \leq \tailsConstant \cdot \epsilon^{\tailsExponent}.\]
Finally, for  $\epsilon \in (0,\min\{\densityA, \volC\})$ we have $\marginalDistribution\left(\{x \in \X: \densityFunction(x)<\epsilon\}\right)=0$.

\noindent Property \ref{quantitativeRangeAssumption} requires us to consider two cases. If $\epsilon \in [w, \fastLimitsThreshold)$ then since $\densityFunction(b) =\densityFunction(d)=1/3>\epsilon$ and for both $\twoMeasuresUnknownNoiseDistIndex \in \{0,1\}$ we have $\regressionFunction^{\twoMeasuresUnknownNoiseDistIndex}(b) \geq 1-\Delta$ and $\regressionFunction^{\twoMeasuresUnknownNoiseDistIndex}(d)=0$ we have 
\begin{align*}
\max\left\lbrace \inf_{x \in \suppMarginalDistribution}\left\lbrace \regressionFunction^{\twoMeasuresUnknownNoiseDistIndex}(x): \densityFunction(x) >\epsilon \right\rbrace, \inf_{x \in \suppMarginalDistribution}\left\lbrace 1-\regressionFunction^{\twoMeasuresUnknownNoiseDistIndex}(x): \densityFunction(x) >\epsilon \right\rbrace\right\rbrace \leq \Delta \leq \fastLimitsConstant \cdot \epsilon^{\fastLimitsExponent}.    
\end{align*}
On the other hand, if $\epsilon \in (0,w)$ then since $\densityFunction(a)=w$ and $\regressionFunction^{\twoMeasuresUnknownNoiseDistIndex}(a)=1$, and $\regressionFunction^{\twoMeasuresUnknownNoiseDistIndex}(d)=0$ we have
\begin{align*}
\max\left\lbrace \inf_{x \in \suppMarginalDistribution}\left\lbrace \regressionFunction^{\twoMeasuresUnknownNoiseDistIndex}(x): \densityFunction(x) >\epsilon \right\rbrace, \inf_{x \in \suppMarginalDistribution}\left\lbrace 1-\regressionFunction^{\twoMeasuresUnknownNoiseDistIndex}(x): \densityFunction(x) >\epsilon \right\rbrace\right\rbrace= 0 \leq \fastLimitsConstant \cdot \epsilon^{\fastLimitsExponent}.    
\end{align*}

\end{proof}

Recall that $\left(\noisy[\Prob]^{\twoMeasuresUnknownNoiseDistIndex}\right)^{\otimes n}$ is the product distribution with $\noisySample \sim \left(\noisy[\Prob]^{\twoMeasuresUnknownNoiseDistIndex}\right)^{\otimes n}$.

\begin{restatable}{lem}{KLBoundOnPairOfProductDistributionsLemma}\label{KLBoundOnPairOfProductDistributionsLemma} $\max\left\lbrace \kullbackLeiblerDivergence\left(\left(\noisy[\Prob]^{0}\right)^{\otimes n},\left(\noisy[\Prob]^{1}\right)^{\otimes n}\right),\kullbackLeiblerDivergence\left(\left(\noisy[\Prob]^{1}\right)^{\otimes n},\left(\noisy[\Prob]^{0}\right)^{\otimes n}\right)\right\rbrace \leq\frac{ 4 n \volA  \Delta^2}{\maxAggregateFlippingProbability} $.
\end{restatable}

\begin{proof} We shall show that $\kullbackLeiblerDivergence\left(\left(\noisy[\Prob]^{0}\right)^{\otimes n},\left(\noisy[\Prob]^{1}\right)^{\otimes n}\right)\leq  (4/\maxAggregateFlippingProbability) \cdot n \volA \cdot \Delta^2$. The proof that\\ $\kullbackLeiblerDivergence\left(\left(\noisy[\Prob]^{1}\right)^{\otimes n},\left(\noisy[\Prob]^{0}\right)^{\otimes n}\right)\leq   (4/\maxAggregateFlippingProbability) \cdot n \volA \cdot \Delta^2$ is similar. Recall that for each ${\twoMeasuresUnknownNoiseDistIndex}\in \{0,1\}$ we let $\noisy[\Prob]^{\twoMeasuresUnknownNoiseDistIndex}$ denote the marginal distribution of $\Prob^{\twoMeasuresUnknownNoiseDistIndex}$ over pairs $(X,\noisy[Y])$ consisting of a feature vector $X \sim \X$ and a corrupted label $\noisy[Y]\in \Y$. We can compute the corrupted regression functions $\noisyRegressionFunction^{\twoMeasuresUnknownNoiseDistIndex}(x)= \Prob^{\twoMeasuresUnknownNoiseDistIndex}[\noisy[Y]|X=x]$ for ${\twoMeasuresUnknownNoiseDistIndex}\in \{0,1\}$ by applying (\ref{writingCorruptedRegressionFunctionInTermsOfRegressionFunctionEq}). Since $\probFlipZeroToOne^0=0$ and $\probFlipOneToZero^0=\maxAggregateFlippingProbability/4$ we have $\noisyRegressionFunction^0(x) = (1-\maxAggregateFlippingProbability/4)\cdot \regressionFunction^0(x)$ for all $x \in \X$. On the other hand, since $\probFlipZeroToOne^1=0$ and $\probFlipOneToZero^1=\Delta+(\maxAggregateFlippingProbability/4)\cdot (1-\Delta)$ we have $\noisyRegressionFunction^1(x) =(1-\maxAggregateFlippingProbability/4)\cdot(1-\Delta)\cdot \regressionFunction^1(x)$ for all $x \in \X$.

We begin by bounding the  Kullback Leibler divergence between $\noisy[\Prob]^0$ and $\noisy[\Prob]^1$ using the fact that $\marginalDistribution(\{a\})=\volA$  and $\noisyRegressionFunction^0(x)=\noisyRegressionFunction^1(x)$ for $x \in \X \backslash \{a\}$,
\begin{align*}
\kullbackLeiblerDivergence\left(\noisy[\Prob]^{0},\noisy[\Prob]^{1}\right)& = \sum_{x \in \X}\sum_{y \in \Y} \noisy[\Prob]^0\left[X = x\hspace{2mm}\&\hspace{2mm}\noisy[Y]=y \right] \log \left(\frac{\noisy[\Prob]^0\left[X = x\hspace{2mm}\&\hspace{2mm}\noisy[Y]=y \right]}{\noisy[\Prob]^1\left[X = x\hspace{2mm}\&\hspace{2mm}\noisy[Y]=y \right]}\right)\\
& = \sum_{x \in \X} \marginalDistribution(\{x\}) \left( (1-\noisyRegressionFunction^0(x)) \log\left( \frac{1-\noisyRegressionFunction^0(x)}{1-\noisyRegressionFunction^1(x)}\right) + \noisyRegressionFunction^0(x)\log\left( \frac{\noisyRegressionFunction^0(x)}{\noisyRegressionFunction^1(x)}\right)
\right)\\
& = \volA \cdot \left( (1-\noisyRegressionFunction^0(a)) \log\left( \frac{1-\noisyRegressionFunction^0(a)}{1-\noisyRegressionFunction^1(a)}\right) + \noisyRegressionFunction^0(a)\log\left( \frac{\noisyRegressionFunction^0(a)}{\noisyRegressionFunction^1(a)}\right)
\right)\\
& \leq \volA \cdot \frac{\left(\noisyRegressionFunction^0(a)-\noisyRegressionFunction^1(a)\right)^2}{\min \left\lbrace \noisyRegressionFunction^0(a), (1-\noisyRegressionFunction^0(a)),\noisyRegressionFunction^1(a), 1-\noisyRegressionFunction^1(a)\right\rbrace}\leq \frac{4}{\maxAggregateFlippingProbability}\cdot \volA \cdot \Delta^2.
\end{align*}
The second to last inequality follows from the reverse Pinsker's inequality \cite[Lemma 6.3]{csiszar2006context}. The final inequality follows from the fact that $\noisyRegressionFunction^0(a) = (1-{\maxAggregateFlippingProbability}/{4})$ and $\noisyRegressionFunction^1(a)=(1-{\maxAggregateFlippingProbability}/{4})\cdot (1-\Delta)$. Given that  $\left(\noisy[\Prob]^{\twoMeasuresUnknownNoiseDistIndex}\right)^{\otimes n}$ consists of $n$ independent copies of $\noisy[\Prob]^{\twoMeasuresUnknownNoiseDistIndex}$ for $\twoMeasuresUnknownNoiseDistIndex \in \{0,1\}$ we deduce that \[\kullbackLeiblerDivergence\left(\left(\noisy[\Prob]^{0}\right)^{\otimes n},\left(\noisy[\Prob]^{1}\right)^{\otimes n}\right)= n \cdot \kullbackLeiblerDivergence\left(\noisy[\Prob]^{0},\noisy[\Prob]^{1}\right) \leq \frac{ 4 n \volA  \Delta^2}{\maxAggregateFlippingProbability}.\]

\end{proof}

\begin{lem}\label{applyingKLUpperBoundLemmaInPfOfLBWithUnknownNoiseLemma} Suppose that $8 n\volA\cdot \Delta^2 \leq  \maxAggregateFlippingProbability$. Given any $\noisySample$-measureable classifier $\classifier$,
\begin{align*}
\sum_{\twoMeasuresUnknownNoiseDistIndex \in \{0,1\}} \left( \ExpectationOverNoisySamples\left[ \risk\left(\classifier\right) \right] -  \risk\left(\oracle\right)  \right) \geq \frac{\volC\cdot \Delta}{8}.
\end{align*}
\end{lem}

\begin{proof} By Lemma \ref{KLBoundOnPairOfProductDistributionsLemma} combined with $8 n\volA\cdot \Delta^2 \leq  \maxAggregateFlippingProbability$ we have 
\[(2 \log 2)^{-1}\cdot\max\left\lbrace \kullbackLeiblerDivergence\left(\left(\noisy[\Prob]^{0}\right)^{\otimes n},\left(\noisy[\Prob]^{1}\right)^{\otimes n}\right),\kullbackLeiblerDivergence\left(\left(\noisy[\Prob]^{1}\right)^{\otimes n},\left(\noisy[\Prob]^{0}\right)^{\otimes n}\right)\right\rbrace \leq  \frac{1}{2}.\]
We construct an estimator $\hat{T}: \left(\X\times \Y\right)^n \rightarrow \left\lbrace \left(\noisy[\Prob]^{0}\right)^{\otimes n}, \left(\noisy[\Prob]^{1}\right)^{\otimes n}\right\rbrace$ in terms of an arbitrary classifier $\classifier$ as follows. Take $\hat{T}(\noisySample) = \left(\noisy[\Prob]^{\twoMeasuresUnknownNoiseDistIndex}\right)^{\otimes n}$ where $\twoMeasuresUnknownNoiseDistIndex=\classifier(c)$ and $\classifier$ is trained on $\noisySample$. Note that $\regressionFunction^0(c)<1/2$ and $\regressionFunction^1(c)>1/2$. Hence, for each $\Prob^{\twoMeasuresUnknownNoiseDistIndex}$ we have $\oracle(c)=\twoMeasuresUnknownNoiseDistIndex$ for the corresponding Bayes rule. By Birge's variant of Fano's lemma (Lemma \ref{birgesVersionOfFanosLemma}), we have
\begin{align}\label{lowerBoundExpectingCPointToBeNonBayesEq}
\sum_{\twoMeasuresUnknownNoiseDistIndex \in \{0,1\}}  \ExpectationOverNoisySamples\left[\one\left\lbrace \classifier(c) \neq \oracle(c) \right\rbrace  \right]  =\sum_{\twoMeasuresUnknownNoiseDistIndex \in \{0,1\}} \left(\noisy[\Prob]^{\twoMeasuresUnknownNoiseDistIndex}\right)^{\otimes n}\left[ \hat{T}\left( \noisySample\right) \neq \left(\noisy[\Prob]^{\twoMeasuresUnknownNoiseDistIndex}\right)^{\otimes n}\right] \geq \frac{1}{4},
\end{align}
where the expectation $\ExpectationOverNoisySamples$ is taken over all samples $\noisySample =\{ (X_i,\noisy[Y]_i)\}_{i \in [n]}$ with $\{ (X_i,Y_i,\noisy[Y]_i)\}_{i \in [n]}$ generated i.i.d. from $\Prob^{\twoMeasuresUnknownNoiseDistIndex}$. To complete the proof of the lemma we note that for both $\twoMeasuresUnknownNoiseDistIndex \in \{0,1\}$ we have $\left| 2\regressionFunction^{\twoMeasuresUnknownNoiseDistIndex}(c)-1\right| = \Delta/(2-\Delta) \geq \Delta/2$. Hence, for $\twoMeasuresUnknownNoiseDistIndex \in \{0,1\}$ and any $\decisionRule \in \setOfDecisionRules$ we have
\begin{align*}
\risk\left(\decisionRule\right) - \risk(\oracle) & = \int \left|2\regressionFunction^{\twoMeasuresUnknownNoiseDistIndex}(x)-1\right| \cdot \one \left\lbrace \decisionRule(x) \neq \oracle(x)\right\rbrace d\marginalDistribution(x)\\
& \geq \marginalDistribution(\{c\}) \cdot \left| 2\regressionFunction^{\twoMeasuresUnknownNoiseDistIndex}(c)-1\right| \cdot \one\left\lbrace \classifier(c) \neq \oracle(c) \right\rbrace\\ &\geq \volC \cdot\frac{\Delta}{2}\cdot \one\left\lbrace \classifier(c) \neq \oracle(c) \right\rbrace.
\end{align*}
Combining with (\ref{lowerBoundExpectingCPointToBeNonBayesEq}) completes the proof of the lemma.
\end{proof}

\paragraph{Proof of Proposition \ref{lowerBoundWithUnknownNoise}}  To prove the proposition we choose parameters $\left(\Delta,r, \volA, \volC,\densityA \right) \in (0,1/6)^5$ so as to maximise the lower bound $v\cdot \Delta/8$ whilst satisfying the conditions of Lemma \ref{propertiesCheckLemmaForUnknownNoiseLB} along with the condition $8 n\volA\cdot \Delta^2 \leq  \maxAggregateFlippingProbability$ from Lemma \ref{applyingKLUpperBoundLemmaInPfOfLBWithUnknownNoiseLemma}. We define $\Delta = 6^{-(1+\frac{1}{\marginExponent}+\fastLimitsExponent )} \cdot \maxAggregateFlippingProbability\cdot ({2n})^{-\frac{\fastLimitsExponent\holderExponent}{\fastLimitsExponent(2\holderExponent+\dimensionExponent)+\holderExponent}}$, $r = \Delta^{\frac{1}{\holderExponent}}$, $\volA = \Delta^{\frac{\holderExponent+\fastLimitsExponent \dimensionExponent}{\fastLimitsExponent \holderExponent}}$, $\volC = \Delta^{\marginExponent}$ and $\densityA=\Delta^{\frac{1}{\fastLimitsExponent}}$. It follows that $\left(\Delta,r, \volA, \volC,\densityA \right) \in (0,1/6)^5$. Moreover, one can then verify that the conditions of Lemma \ref{propertiesCheckLemmaForUnknownNoiseLB} hold, so for both $\twoMeasuresUnknownNoiseDistIndex \in \{0,1\}$ we have $\Prob^{\twoMeasuresUnknownNoiseDistIndex} \in \distributionClass\left(\Gamma\right)$. In addition, we have $2nu \cdot \Delta^2 \leq 1$, so by Lemma \ref{applyingKLUpperBoundLemmaInPfOfLBWithUnknownNoiseLemma},
\begin{align*}
\sup_{\Prob \in \distributionClass(\Gamma)} \left( \ExpectationOverNoisySamples\left[ \risk\left(\classifier\right) \right] -  \risk\left(\oracle\right)  \right) & \geq \frac{1}{2}\sum_{\twoMeasuresUnknownNoiseDistIndex \in \{0,1\}} \left( \ExpectationOverNoisySamples\left[ \risk\left(\classifier\right) \right] -  \risk\left(\oracle\right)  \right) \\
&\geq \frac{\volC\cdot \Delta}{16} = \frac{\Delta^{\marginExponent+1}}{16}= c_1 \cdot n^{-\frac{\fastLimitsExponent\holderExponent(\marginExponent+1)}{\fastLimitsExponent(2\holderExponent+\dimensionExponent)+\holderExponent}},
\end{align*}
where $c_1$ is determined by $\Gamma$. This completes the proof of the Proposition \ref{lowerBoundWithUnknownNoise}.
\QEDA

To complete the proof of Theorem \ref{mainMinimaxThmLowerBound} we will combine Proposition \ref{lowerBoundWithUnknownNoise} with Proposition \ref{lowerBoundWithKnownNoise} in the next section.

\subsection{A lower bound for uncorrupted data}\label{LBForUncorruptedDataAppendix}

 In this section we prove Proposition \ref{lowerBoundWithKnownNoise} which component corresponds to the difficulty of the core classification problem which would have been present even if the learner had access to clean labels. We can then complete the proof of Theorem \ref{mainMinimaxThmLowerBound}.

\begin{prop}\label{lowerBoundWithKnownNoise} Take $\Gamma = (\maxAggregateFlippingProbability,\dimensionExponent,(\marginExponent,\marginConstant),(\holderExponent,\holderConstant), (\tailsExponent,\tailsThreshold,\tailsConstant), ( \fastLimitsExponent, \fastLimitsThreshold, \fastLimitsConstant) )$ consisting of exponents $\marginExponent \in [0,\infty)$, $\holderExponent \in (0,1]$, $\dimensionExponent \in [\marginExponent \holderExponent,\infty)$, $\tailsExponent \in (0,\infty)$, $\fastLimitsExponent \in (0, \infty)$ with constants $\marginConstant$, $\holderConstant, \tailsConstant, \fastLimitsConstant \geq 1$, and $\tailsThreshold \in (0,1/24)$, $\fastLimitsThreshold \in (0,1/3)$. There exists a constant $c_0(\Gamma)$, depending solely upon $\Gamma$, such that for any $n \in \N$ and any classifier $\classifier$ which is measurable with respect to the corrupted sample $\noisySample$, there exists a distribution $\Prob \in \distributionClass(\Gamma)$ with $\noisyDistribution=\cleanDistribution$ such that 
\begin{align*}
\ExpectationOverNoisySamples\left[ \risk\left(\classifier\right) \right] -  \risk\left(\oracle\right) \geq c_0(\Gamma) \cdot n^{-\frac{\holderExponent \tailsExponent(\marginExponent+1)}{\tailsExponent(2\holderExponent+\dimensionExponent)+\marginExponent\holderExponent} }.
\end{align*}
\end{prop}

To prove Proposition \ref{lowerBoundWithKnownNoise}  we will construct a family of measure distributions contained within $\distributionClass(\Gamma)$. We will then use an important lemma of Audibert \cite[Lemma 5.1]{audibert2004classification} to deduce the lower bound.

\newcommand{\A}{\mathcal{A}}
\newcommand{\G}{\mathcal{G}}

\paragraph{\textbf{Families of measures}} Take parameters $l \in \N$ with $l \geq 2$, $w \leq 1/3$, $\Delta \leq 1$, $m \leq 2^{l-1}$, whose value will be made precise below. We let $\A=\left\lbrace \vec{a}=(a_q)_{q \in [l]} \in \{0,1\}^l\right\rbrace$ and choose $\A^{\sharp}\subset \left\lbrace \vec{a}=(a_q)_{q \in [l]} \in \A: a_l = 1\right\rbrace$ with $\left|\A^{\sharp}\right|=m$. This is possible since $m \leq 2^{l-1}$. Given $\vec{a}^0=(a^0_q)_{q \in [l]}$, $\vec{a}^1=(a_q^1)_{q \in [l]} \in \A$ we let $|\vec{a}^0\wedge\vec{a}^1|:= \max\left\lbrace {k \in [l]}: a^0_q = a^1_q \text{ for }q \leq k\right\rbrace$ denote the length of the largest common substring. Let $\X = \A \cup \{0\} \cup \{1\}$ and define a metric $\dist$ on $\X$ by 
\begin{align*}
\dist(x_0,x_1) = \begin{cases} 2^{-|x_0\wedge x_1|/\dimensionExponent} &\text{ if }x_0, x_1 \in \A\text{ and }x_0 \neq x_1\\
1&\text{ if }x_0 \notin \A\text{ or }x_1 \notin \A\text{ and }x_0 \neq x_1\\
0&\text{ if }x_0 = x_1.
\end{cases}
\end{align*}
One can easily verify that $\dist$ is non-negative, symmetric, satisfies the identity of indiscernibles property and the triangle inequality. We may define a Borel probability measure $\marginalDistribution$ on $\X$ by letting
\begin{align*}
\marginalDistribution(\{x\}) = \begin{cases} \frac{1}{3} &\text{ if } x \in \{0,1\}\\
w \cdot 2^{-l} &\text{ if } x \in \A^{\sharp}\\
\frac{1-3 m w \cdot 2^{-l}}{3(2^l-m)}&\text{ if }x \in \A\backslash \A^{\sharp}.
\end{cases}
\end{align*}
One can easily verify that $\marginalDistribution$ extends to a well-defined probability measure on $\X$ and for $x \in \A\backslash \A^{\sharp}$ we have $\marginalDistribution(\{x\}) \geq (1/6) \cdot 2^{-l}$. Finally, we define a density function $\densityFunction:\X \rightarrow (0,1)$ by
\begin{align*}
\densityFunction(x) = \begin{cases} \frac{1}{3} &\text{ if } x \in \{0,1\}\\
\frac{w}{8} &\text{ if } x \in \A^{\sharp}\\
\frac{1}{24} &\text{ if }x \in \A\backslash \A^{\sharp}.
\end{cases}
\end{align*}
We now let $\G=\left\lbrace g:\A^{\sharp}\rightarrow \{-1,+1\}\right\rbrace$. For each $g \in \G$ we define an associated regression function $\regressionFunction^g:\X \rightarrow [0,1]$ by
\begin{align*}
\regressionFunction^g(x) = \begin{cases} 0 &\text{ if } x =0 \\
 1 &\text{ if } x =1 \\
\frac{1+\Delta \cdot g(x)}{2} &\text{ if } x \in \A^{\sharp}\\
\frac{1}{2}&\text{ if }x \in \A\backslash \A^{\sharp}.
\end{cases}
\end{align*}
Finally, we define distributions $\Prob^g$ on triples $(X,Y,\noisy[Y]) \in \X\times \Y^2$ for each $g \in \G$ as follows:
\begin{enumerate}
    \item Let $\marginalDistribution$ be the marginal distribution over $X$ i.e. $\Prob^g[X \in A] = \marginalDistribution(A)$ for $A\subset \X$;
    \item Let $\regressionFunction^g$ be the regression function i.e. $\Prob^g[Y|X=x] = \regressionFunction^g(x)$ for $x \in \X$;
    \item Take $\Prob^g[\noisy[Y]=Y]=1$.
\end{enumerate}
Note that $\Prob^g[\noisy[Y]=Y]=1$ implies that $\Prob^g_{\text{clean}}=\Prob^g_{\text{corr}}$, where $\Prob^g_{\text{clean}}$ denotes the marginal over $(X,Y)$ and $\Prob^g_{\text{corr}}$ denotes the marginal over $(X,\noisy[Y])$. The following Lemma gives conditions under which $\Prob^g\in \distributionClass(\Gamma)$ for all $g \in \G$.

\begin{restatable}{lem}{propertiesCheckLemmaForNoNoiseLB}\label{propertiesCheckLemmaForNoNoiseLB} For all $g \in \G$ the measure $\Prob^g$ satisfy the following properties:
\begin{enumerate}[label=(\Alph*)]
    \item $\Prob^{g}$ satisfies Assumption \ref{classConditionalNoiseAssumption};
    \item $\Prob^g$ satisfies Assumption \ref{marginAssumption} parameters $(\marginExponent, \marginConstant)$ whenever $ m \cdot w \cdot 2^{-l} \leq \marginConstant \cdot (\Delta/2)^{\marginExponent}$;
    \item $\regressionFunction^{g}$ satisfies Assumption \ref{holderAssumption} with parameters $(\holderExponent, \holderConstant)$ whenever $\Delta \leq \holderConstant\cdot 2^{-(l-1)\cdot(\holderExponent/\dimensionExponent)}$;
    \item $\marginalDistribution$ satisfies Assumption \ref{minimalMassAssumption} with parameters $(\dimensionExponent, \densityFunction)$;
    \item $\marginalDistribution$ satisfies Assumption \ref{tailAssumption} with parameters $(\tailsExponent, \tailsConstant, \tailsThreshold, \densityFunction)$ when $\tailsThreshold \leq \frac{1}{24}$ \& $\frac{m \cdot w}{2^l} \leq \tailsConstant \cdot \left(\frac{w}{8}\right)^{\tailsExponent}$;
    \item $\Prob^{\twoMeasuresUnknownNoiseDistIndex}$ satisfies Assumption \ref{quantitativeRangeAssumption} with parameters $(\fastLimitsExponent, \fastLimitsConstant,\fastLimitsThreshold,\densityFunction)$ whenever $\fastLimitsThreshold\leq 1/3$. 
\end{enumerate}
\end{restatable}

\begin{proof}
\noindent Property \ref{classConditionalNoiseAssumption} is immediate from the fact that $\Prob^g[\noisy[Y]=Y]=1$.
\noindent Property \ref{marginAssumption} follows from the fact the construction of $\regressionFunction^g$. Indeed, for $\epsilon \in [\Delta/2,1)$ we have
\begin{align*}
\marginalDistribution\left(\left\lbrace x \in \X: 0<\left| \regressionFunction^g(x)-\frac{1}{2}\right|<\epsilon\right\rbrace\right) = \marginalDistribution\left(\A^{\sharp}\right) = m \cdot w \cdot 2^{-l} \leq \marginConstant \cdot (\Delta/2)^{\marginExponent} \leq \marginConstant \cdot \epsilon^{\marginExponent}.
\end{align*}
However, if $\epsilon \in (0,\Delta/2)$ then $\left\lbrace x \in \X: 0<\left| \regressionFunction^g(x)-\frac{1}{2}\right|<\epsilon\right\rbrace =\emptyset$.

\noindent Property \ref{holderAssumption} follows from the fact that if $x_0 \neq x_1 \in \X$ satisfy $\dist(x_0,x_1)<1$ then we must have $x_0, x_1 \in \A$ so
\begin{align*}
\left|\regressionFunction^g(x_0)-\regressionFunction^g(x_1)\right| \leq \Delta \leq \holderConstant\cdot 2^{-(l-1)\cdot(\holderExponent/\dimensionExponent)}\leq \holderConstant \cdot \dist(x_0,x_1)^{\holderExponent}.
\end{align*}

\noindent Property \ref{minimalMassAssumption} requires four cases. The first case is straightforward: If $x \in \{0,1\}$ then for any $r\in (0,1)$ we have $\marginalDistribution(B_r(x)) = \frac{1}{3}=\densityFunction(x)\geq \densityFunction(x) \cdot r^{\dimensionExponent}$. Next we consider $x=(a_q)_{q \in [l]} \in \A$ with $r \in \left(2^{(1-l)/\dimensionExponent},1\right)$. Choose an integer $p \in [l-1]$ with $2^{-p/\dimensionExponent}< r \leq 2^{(1-p)/\dimensionExponent}$.  Then by the construction of the metric $\dist$ we have
\begin{align*}
B_r(x) & \supset \left\lbrace \tilde{\vec{a}} \in \A: \tilde{a}_q = a_q\text{ for all } q\leq p\right\rbrace\\
& \supset \left\lbrace \tilde{\vec{a}} \in \A: \tilde{a}_q = a_q\text{ for all } q\leq p \text{ and }a_l = 0\right\rbrace\\
& = \left\lbrace \tilde{\vec{a}} \in \A\backslash \A^{\sharp}: \tilde{a}_q = a_q\text{ for all } q\leq p \text{ and }a_l = 0\right\rbrace.
\end{align*}
Moreover, the above set is of cardinality $2^{l-p-1}$. Hence, the cardinality of $B_r(x) \cap \left(\A\backslash \A^{\sharp}\right)$ is at least $2^{l-p-1}$. Since we have $\marginalDistribution(\{\tilde{\vec{a}}\}) \geq (1/6) \cdot 2^{-l}$ for $\tilde{\vec{a}} \in \A \backslash \A^{\sharp}$ it follows that
\begin{align*}
\marginalDistribution\left(B_r(x)\right) \geq \left(2^{l-p-1}\right) \cdot \left( (1/6) \cdot 2^{-l}\right) = \frac{1}{24}\cdot 2^{(1-p)}\geq \frac{1}{24} \cdot r^{\dimensionExponent} \geq \densityFunction(x) \cdot r^{\dimensionExponent}.
\end{align*}
The third case is where $x \in \A^{\sharp}$ and $r \in \left(0,2^{(1-l)/\dimensionExponent}\right]$, in which case we have
\begin{align*}
\marginalDistribution\left(B_r(x)\right) \geq \marginalDistribution\left(\{x\}\right) = w \cdot 2^{-l} \geq \frac{w}{2}\cdot r^{\dimensionExponent} \geq \densityFunction(x) \cdot r^{\dimensionExponent}.
\end{align*}
Finally, we consider $x \in \A\backslash \A^{\sharp}$ and $r \in \left(0,2^{(1-l)/\dimensionExponent}\right]$, in which case 
\begin{align*}
\marginalDistribution\left(B_r(x)\right) \geq \marginalDistribution\left(\{x\}\right) = \frac{1}{6} \cdot 2^{-l} \geq \frac{1}{12}\cdot r^{\dimensionExponent} \geq \densityFunction(x) \cdot r^{\dimensionExponent}.
\end{align*}

\noindent Property \ref{tailAssumption} requires two cases. If $\epsilon \in (w/8,\tailsThreshold)$ then
\begin{align*}
\marginalDistribution\left(\left\lbrace x \in \X: \densityFunction(x)<\epsilon\right\rbrace\right) =\marginalDistribution\left(\A^{\sharp}\right) = m \cdot w \cdot 2^{-l} \leq \tailsConstant \cdot \left(\frac{w}{8}\right)^{\tailsExponent}\leq  \tailsConstant \cdot \epsilon^{\tailsExponent}.   
\end{align*}
However, if $\epsilon \leq w/8$ then $\left\lbrace x \in \X: \densityFunction(x)<\epsilon\right\rbrace=\emptyset$.

\noindent Property \ref{quantitativeRangeAssumption} is straightforward since $\regressionFunction^{\fastLimitsExponent}(0)=0$ and $\regressionFunction^{\fastLimitsExponent}(1)=1$, so for $\epsilon \in \left(0, \fastLimitsThreshold\right)$ we have
\[\max\left\lbrace \inf_{x \in \suppMarginalDistribution}\left\lbrace \regressionFunction(x): \densityFunction(x) >\epsilon \right\rbrace, \inf_{x \in \suppMarginalDistribution}\left\lbrace 1-\regressionFunction(x): \densityFunction(x) >\epsilon \right\rbrace\right\rbrace=0 \leq \fastLimitsConstant\cdot \epsilon^{\fastLimitsExponent}.\]
\end{proof}

We now recall some useful terminology due to \cite{audibert2004classification}.

\begin{defn}[Probability hypercube]\label{defOfAudibertProbHypercube} Take $m \in \N$, $v \in (0,1]$ and $\Delta \in (0,1]$. Suppose that $\X$ is a metric space with a partition  $\left\lbrace \X_0, \cdots, \X_m \right\rbrace$ into $m+1$ disjoint sets. Let $\marginalDistribution$ be a Borel measure on $\X$ such that for each $j \in \{1,\cdots, m\}$, $\marginalDistribution(\X_j)=v$. Let $\xi: \X\rightarrow [0,1]$ be a function such that for each $j \in \{1,\cdots, m\}$ and $x \in \X_j$, $\xi(x)= \Delta$. Let $\sigma_0$ and for each $\vec{\sigma} = \left(\sigma_j\right)_{j \in [m]} \in \{-1,+1\}^m$ we define an associated regression function $\regressionFunction_{\vec{\sigma}}:\X \rightarrow [0,1]$ by
\begin{align*}
\regressionFunction_{\vec{\sigma}}(x) = \frac{1+\sigma_j \cdot \xi(x)}{2} \text{  for  }x \in \X_j.
\end{align*}
For each $\vec{\sigma} = \left(\sigma_j\right)_{j \in [m]} \in \{-1,+1\}^m$ we let $\overline{\Prob}^{\vec{\sigma}}$ be the unique probability measure on $\X \times \Y$ such that $\overline{\Prob}^{\vec{\sigma}}[X \in A] = \marginalDistribution(A)$ for all Borel sets $A\subset \X$ and $\overline{\Prob}^{\vec{\sigma}}[Y=1|X=x]= \regressionFunction_{\vec{\sigma}}(x)$ for $x \in \X$. A family of distributions $\left\lbrace \overline{\Prob}^{\vec{\sigma}} : \vec{\sigma} = \left(\sigma_j\right)_{j \in [m]} \in \{-1,+1\}^m \right\rbrace$ of this form is referred to as a $(m,v,\Delta)$-hypercube.
\end{defn}

We shall utilise the following useful variant of Assouad's lemma from \cite[Lemma 5.1]{audibert2004classification}.

\begin{lemma}[Audibert's lemma]\label{audibertsVersionOfAssouadsLemmaForClassification} Let $\overline{\distributionClass}$ be a set of distributions containing a $(m,v,\Delta)$. Then for any classifier $\classifier$ measureable with respect to the sample $\sample= \left\lbrace (X_i,Y_i)\right\rbrace$ there exists a distribution $\overline{\Prob} \in \overline{\distributionClass}$ with 
\begin{align*}
\overline{\E}^{\otimes n}\left[ \risk\left(\classifier\right) \right] -  \risk\left(\oracle\right) \geq \frac{1-\Delta\cdot \sqrt{n v}}{2} \cdot (m v \Delta),
\end{align*}
where $\overline{\E}^{\otimes n}$ denotes the expectation over all samples $\sample = \left\lbrace (X_i,Y_i)\right\rbrace \in \left(\X\times \Y\right)^n$ with $(X_i, Y_i) \sim \overline{\Prob}$ sampled independently.
\end{lemma}

We are now in a position to complete the proof of  Proposition \ref{lowerBoundWithKnownNoise}.

\paragraph{Proof of Proposition \ref{lowerBoundWithKnownNoise}} First note that for any class of distributions $\distributionClass$ the minimax rate,
\[\inf_{\classifier}\left\lbrace \sup_{\Prob \in \distributionClass}\left\lbrace 
\ExpectationOverNoisySamples\left[ \risk\left(\classifier\right) \right] -  \risk\left(\oracle\right) \right\rbrace \right\rbrace,\]
 is monotonically non-increasing with $n$. Hence, it suffices to show that there exists $N_0 \in \N$ and $C_0 \in (0,\infty)$, both depending solely upon $\Gamma$, such that for any $n \in \N$ and any classifier $\classifier$, measurable with respect to $\noisySample$, there exists $\Prob \in \distributionClass(\Gamma)$ with $\cleanDistribution = \noisyDistribution$ and
\begin{align}\label{keyClaimInlowerBoundWithKnownNoise} 
\ExpectationOverNoisySamples\left[ \risk\left(\classifier\right) \right] -  \risk\left(\oracle\right) \geq C_0 \cdot n^{\frac{\holderExponent \tailsExponent(\marginExponent+1)}{\tailsExponent(2\holderExponent+\dimensionExponent)+\marginExponent\holderExponent} }.
\end{align}
Proposition \ref{lowerBoundWithKnownNoise} will then follow with an appropriately modified constant. To prove the claim (\ref{keyClaimInlowerBoundWithKnownNoise}) consider the class of measures $\left\lbrace \Prob^g\right\rbrace_{g \in \G}$ with some parameters $l \in \N$ with $l \geq 2$, $w \leq 1/3$, $\Delta \leq 1$, $m \leq 2^{l-1}$ to be specified shortly. We observe that the set $\left\lbrace \Prob^g_{\text{clean}}\right\rbrace_{g \in \G}$ of corresponding clean distributions is an $(m,v,\Delta)$ hyper cube with $v = w\cdot 2^{-l}$. To see this first let $\left\lbrace \X_j \right\rbrace_{j = 1}^m$ be a partition of $\A^{\sharp}$ into singletons and let $\X_0 = \X \backslash \A^{\sharp}$. Note that this is possible since $\A^{\sharp}$ is of cardinality $m$. Moreover, we have $\marginalDistribution(\X_j) =v = w\cdot 2^{-l}$ for each $j \in [m]$. Define $\xi:\X \rightarrow [0,1]$ by
\begin{align*}
\xi(x) = &\begin{cases} -1 &\text{ if } x= 0\\
+1 &\text{ if } x=1\\
\Delta &\text{ if } x \in \A^{\sharp}\\
0 &\text{ if }x \in \A\backslash \A^{\sharp}.
\end{cases}
\end{align*}
It follows that the set of clean distributions $\left\lbrace \Prob^g_{\text{clean}}\right\rbrace_{g \in \G}$ is precisely the $(m,v,\Delta)$ constructed in Definition \ref{defOfAudibertProbHypercube}. Hence, by applying Lemma \ref{audibertsVersionOfAssouadsLemmaForClassification} we see that for some $\Prob \in \left\lbrace \Prob^g\right\rbrace_{g \in \G}$ we have
\begin{align}\label{audibertsLemmaAppliedLB}
\ExpectationOverNoisySamples\left[ \risk\left(\classifier\right) \right] -  \risk\left(\oracle\right)=\ExpectationOverCleanSamples\left[ \risk\left(\classifier\right) \right] -  \risk\left(\oracle\right) \geq \frac{1-\Delta\cdot \sqrt{n v}}{2} \cdot (m v \Delta).
\end{align}
Here we have used the fact that for $g \in \G$ we have $\Prob^g_{\text{clean}}=\Prob^g_{\text{corr}}$. 

To complete the proof we select the parameters $l \in \N$ with $l \geq 2$, $w \leq 1/3$, $\Delta \leq 1$, $m \leq 2^{l-1}$ so as to approximately maximise the lower bound in (\ref{audibertsLemmaAppliedLB}) whilst satisfying the conditions of Lemma \ref{propertiesCheckLemmaForNoNoiseLB}. To do so we take $l =  \left\lceil \frac{\dimensionExponent \tailsExponent}{\tailsExponent(2\holderExponent+\dimensionExponent)+\marginExponent\holderExponent} \cdot \frac{\log(2n)}{\log 2}\right\rceil+1$ and $\Delta = \left(2^{-l}\right)^{\frac{\holderExponent}{\dimensionExponent}}$ so that $l \geq 2$, $\Delta \leq 1$ and 
\begin{align}\label{upperLowerBoundsDeltaPfOfNoiseFreeLB}
\frac{1}{4^{{\frac{\holderExponent}{\dimensionExponent}}}}\cdot \left(\frac{1}{2n}\right)^{\frac{\holderExponent \tailsExponent}{\tailsExponent(2\holderExponent+\dimensionExponent)+\marginExponent\holderExponent} }\leq \Delta \leq \left(\frac{1}{2n}\right)^{\frac{\holderExponent \tailsExponent}{\tailsExponent(2\holderExponent+\dimensionExponent)+\marginExponent\holderExponent} }.
\end{align}
Let $w= \frac{1}{3} \cdot \Delta^{\frac{\marginExponent}{\tailsExponent}}$ and $m = \left\lfloor \min\left\lbrace \frac{1}{2},\frac{1}{2^{\marginExponent}},\frac{1}{24^{\tailsExponent}}\right\rbrace \cdot \Delta^{-\frac{\marginExponent\holderExponent+\tailsExponent(\dimensionExponent-\marginExponent\holderExponent)}{\tailsExponent \holderExponent}}\right\rfloor$. One can verify that with these choices we have $m\leq 2^{l-1}$, $\frac{m\cdot w}{2^l}\leq \min\left\lbrace \left(\frac{\Delta}{2}\right)^{\marginExponent},\left(\frac{w}{8}\right)^{\tailsExponent}\right\rbrace$ and $\Delta \leq 2^{-(l-1)\cdot (\holderExponent/\dimensionExponent)}$. Thus, by Lemma \ref{propertiesCheckLemmaForNoNoiseLB} we have  $\Prob^g \in \distributionClass\left(\Gamma\right)$ for all $g \in \G$. 

Since $\marginExponent \cdot \holderExponent \leq \dimensionExponent$ and $\Delta$ decreases towards zero, there exists $N_0 \in \N$, determined by $\Gamma$, such that for all $n \geq N_0$ we have $\min\left\lbrace \frac{1}{2},\frac{1}{2^{\marginExponent}},\frac{1}{24^{\tailsExponent}}\right\rbrace \cdot \Delta^{-\frac{\marginExponent\holderExponent+\tailsExponent(\dimensionExponent-\marginExponent\holderExponent)}{\tailsExponent \holderExponent}} \geq 2$. We have $v=w\cdot 2^{-l}=\frac{1}{3} \cdot \Delta^{\frac{\marginExponent \holderExponent+\tailsExponent \dimensionExponent}{\tailsExponent\holderExponent}}$, so by (\ref{upperLowerBoundsDeltaPfOfNoiseFreeLB})
$\Delta^2 \cdot n \cdot v \leq 1/6$. Thus, by (\ref{audibertsLemmaAppliedLB}) we see that there exists a constants $K_j$, depending only upon $\dimensionExponent,(\marginExponent,\marginConstant),(\holderExponent,\holderConstant), (\tailsExponent,\tailsThreshold,\tailsConstant), ( \fastLimitsExponent, \fastLimitsThreshold, \fastLimitsConstant)$, such that for all $n \geq N_0$ and any $\noisySample$ measureable classifier there exists a distribution $\Prob \in \left\lbrace \Prob^g \right\rbrace_{ g \in \G}\subset \distributionClass\left(\Gamma\right)$ with
\begin{align*}
\ExpectationOverNoisySamples\left[ \risk\left(\classifier\right) \right] -  \risk\left(\oracle\right) &\geq \KConstant \cdot (mv\Delta) \geq  \KConstant \cdot \Delta^{-\frac{\marginExponent\holderExponent+\tailsExponent(\dimensionExponent-\marginExponent\holderExponent)}{\tailsExponent \holderExponent}} \cdot  \Delta^{\frac{\marginExponent \holderExponent+\tailsExponent \dimensionExponent}{\tailsExponent\holderExponent}} \cdot \Delta \\& = \addtocounter{Kcounter}{-1} \KConstant \cdot \Delta^{1+\marginExponent} \geq \KConstant \cdot \left(\frac{1}{n}\right)^{\frac{\holderExponent \tailsExponent(\marginExponent+1)}{\tailsExponent(2\holderExponent+\dimensionExponent)+\marginExponent\holderExponent} }.
\end{align*}
This proves the claim (\ref{keyClaimInlowerBoundWithKnownNoise}) and completes the proof of Proposition \ref{lowerBoundWithKnownNoise}.
\QEDA

We can now complete the proof of Theorem \ref{mainMinimaxThmLowerBound}.
\paragraph{Proof of Theorem \ref{mainMinimaxThmLowerBound}}
Theorem \ref{mainMinimaxThmLowerBound} follows immediately from Propositions \ref{lowerBoundWithUnknownNoise} and \ref{lowerBoundWithKnownNoise}.
\QEDA

\section{Proof of the upper bound}\label{proofOfHPUpperBoundAppendix}

In this section we prove Theorem \ref{mainUpperBoundThm}. We with an elementary lemma.

\begin{lemma}\label{elementaryRatioLemma} Suppose that $\estProbFlipZeroToOne, \estProbFlipOneToZero \in \left[0,1\right)$ with $\estProbFlipZeroToOne+\estProbFlipOneToZero<1$. Let $\noisy[{\est[\regressionFunction]}]:\X \rightarrow [0,1]$ be an estimate of $\noisyRegressionFunction$ and define $\est[\regressionFunction] :\X \rightarrow [0,1]$ by $\est[\regressionFunction](x):= \left(\noisy[{\est[\regressionFunction]}](x)-\estProbFlipZeroToOne\right)/\left(1-\estProbFlipZeroToOne-\estProbFlipOneToZero\right)$. Suppose that $\probFlipZeroToOne+\probFlipOneToZero<1$ and $\max\left\lbrace\left| \estProbFlipZeroToOne-\probFlipZeroToOne \right|, \left| \estProbFlipOneToZero-\probFlipOneToZero \right| \right\rbrace  \leq \left(1-\probFlipZeroToOne-\probFlipOneToZero\right)/4$. Then for all $x \in \X$ we have
\begin{align*}
\left| \est[\regressionFunction](x)-\regressionFunction(x)\right| \leq 8 \cdot \left(1-\probFlipZeroToOne-\probFlipOneToZero\right)^{-1} \cdot \max\left\lbrace \left| \estNoisyRegressionFunction(x) - \noisyRegressionFunction(x)\right|, \left| \estProbFlipZeroToOne-\probFlipZeroToOne \right|, \left| \estProbFlipOneToZero-\probFlipOneToZero \right| \right\rbrace.
\end{align*}
\end{lemma}
\begin{proof}
An elementary computation shows that given $\hat{a},a \in [-1,1]$ and $\hat{b},b \in \left(0,\infty\right)$ with $|\hat{b}-b| \leq b/2$ and $|a/b|\leq 1$ we have
\begin{align*}
\left|\frac{\hat{a}}{\hat{b}}-\frac{a}{b}\right| \leq \frac{4}{b}\cdot \max\left\lbrace |\hat{a}-a|, |\hat{b}-b| \right\rbrace.
\end{align*}
The lemma now follows from $\noisyRegressionFunction(x) = \left(1-\probFlipZeroToOne-\probFlipOneToZero\right) \cdot \regressionFunction(x)+ \probFlipZeroToOne$ (\ref{writingCorruptedRegressionFunctionInTermsOfRegressionFunctionEq}) by taking $\hat{a} = \noisy[\hat{\regressionFunction}](x)-\estProbFlipZeroToOne$, $a =\noisyRegressionFunction(x)-\probFlipZeroToOne$,\hspace{3mm}   $\hat{b} = 1-\estProbFlipZeroToOne-\estProbFlipOneToZero$ and $b = 1-\probFlipZeroToOne-\probFlipOneToZero$.
\end{proof}

\paragraph{Proof of Theorem \ref{mainUpperBoundThm}} Throughout the proof we let $K_l$ denote constants whose value depends solely upon
$\dimensionExponent,(\marginExponent,\marginConstant),(\holderExponent,\holderConstant), (\tailsExponent,\tailsThreshold,\tailsConstant), ( \fastLimitsExponent, \fastLimitsThreshold, \fastLimitsConstant)$. First we introduce a data-dependent subset $\goodDeltaSet\subset \X$ consisting of points where $\noisy[{\est[\regressionFunction]}](x)$ provides a good estimate of $\noisyRegressionFunction(x)$,
\begin{align*}
\goodDeltaSet:= \left\lbrace x \in \X: \left| \noisy[{\est[\regressionFunction]}](x)-\noisyRegressionFunction(x) \right|\leq (8 \sqrt{2}) \cdot   {\holderConstant}^{\frac{\dimensionExponent}{2\holderExponent+\dimensionExponent}} \cdot\left( \frac{\log(12n/\delta^2)}{\densityFunction(x)\cdot n} \right)^{\frac{\holderExponent}{2\holderExponent+\dimensionExponent}}\right\rbrace.
\end{align*}
By (\ref{writingCorruptedRegressionFunctionInTermsOfRegressionFunctionEq}) combined with the H\"{o}lder assumption (Assumption \ref{holderAssumption}) on $\regressionFunction$, $\noisyRegressionFunction$ also satisfies the H\"{o}lder assumption with parameters $(\holderExponent, \holderConstant)$. By Theorem \ref{knnRegressionBoundLepskiK} we have $\ExpectationOverNoisySamples\left[\one\left\lbrace x \notin \goodDeltaSet\right\rbrace\right] \leq \delta^2/3$, for each $x \in \suppMarginalDistribution$, where $\ExpectationOverNoisySamples$ denote the expectation over the corrupted sample $\noisySample$. Hence, by Fubini's theorem we have
\begin{align*}
\ExpectationOverNoisySamples\left[\marginalDistribution\left(\X\backslash \goodDeltaSet \right)\right]=\ExpectationOverNoisySamples\left[\int\one\left\lbrace x \notin \goodDeltaSet\right\rbrace d\marginalDistribution(x)\right]=\int\ExpectationOverNoisySamples\left[\one\left\lbrace x \notin \goodDeltaSet\right\rbrace\right]d\marginalDistribution(x) \leq \delta^2/3.
\end{align*}
Hence, by Markov's inequality we have $\marginalDistribution\left(\X\backslash \goodDeltaSet \right)\leq \delta$ with probability at most $1-\delta/3$ over $\noisySample$. Now let $\epsDeltaN:= \log(n/\delta)/n$. Recall that $M(f)$ denotes the maximum of an arbitrary function $f$. By  (\ref{writingCorruptedRegressionFunctionInTermsOfRegressionFunctionEq}) we have $\probFlipZeroToOne = 1-M(1-\noisyRegressionFunction)$ and $\probFlipOneToZero = 1-M(\noisyRegressionFunction)$. Hence, by Theorem \ref{funcMaxThm} both of the following bounds hold with probability at least $1-2\delta/3$ over $\noisySample$,
\begin{align} 
\left| \estProbFlipZeroToOne-\probFlipZeroToOne\right| &= \left|\hat{M}_{{n},{\delta/3}}\left(1-\noisy[\regressionFunction]\right) -M(1-\noisyRegressionFunction)\right| &\leq   \KConstant \cdot \epsDeltaN^{\frac{\fastLimitsExponent \holderExponent}{\fastLimitsExponent(2\holderExponent+\dimensionExponent)+\holderExponent}},\nonumber \\ \addtocounter{Kcounter}{-1}
\left| \estProbFlipOneToZero-\probFlipOneToZero\right|  &= \left|\hat{M}_{{n},{\delta/3}}\left(\noisy[\regressionFunction]\right) -M(\noisyRegressionFunction)\right|  &\leq   \KConstant \cdot \epsDeltaN^{\frac{\fastLimitsExponent \holderExponent}{\fastLimitsExponent(2\holderExponent+\dimensionExponent)+\holderExponent}}.\label{noiseProbEstBoundsInMainUpperBound}
\end{align}
Thus, applying the union bound once again we have both $\marginalDistribution\left(\X\backslash \goodDeltaSet\right) \leq \delta$ and the two bounds in (\ref{noiseProbEstBoundsInMainUpperBound}), simultaneously, with probability at least $1-\delta$ over $\noisySample$. Hence, to complete the proof of Theorem \ref{mainUpperBoundThm} it suffices to assume $\marginalDistribution\left(\X\backslash \goodDeltaSet\right) \leq \delta$ and (\ref{noiseProbEstBoundsInMainUpperBound}), and deduce the following bound,
\begin{align}\label{conclusionInRiskBoundMaxFormulation}
\risk\left(\est[\decisionRule]_{n,\delta} \right) - \risk\left( \oracle\right) \leq \frac{ \KConstant}{  \left(1-\probFlipZeroToOne-\probFlipOneToZero\right)^{1+\marginExponent}}  \cdot \max\left\lbrace \epsDeltaN^{\frac{\tailsExponent\holderExponent(\marginExponent+1)}{\tailsExponent(2\holderExponent+\dimensionExponent)+\marginExponent\holderExponent}},\epsDeltaN^{\frac{\fastLimitsExponent \holderExponent(\marginExponent+1)}{\fastLimitsExponent(2\holderExponent+\dimensionExponent)+\holderExponent}} \right\rbrace+\delta.
\end{align}

We can rewrite $\est[\decisionRule]_{n,\delta}:\X\rightarrow \Y$ as $\est[\decisionRule]_{n,\delta}(x)=\one\left\lbrace \est[\regressionFunction](x)\geq 1/2\right\rbrace$, where \[\est[\regressionFunction](x):= \left(\noisy[{\est[\regressionFunction]}](x)-\estProbFlipZeroToOne\right)/\left(1-\estProbFlipZeroToOne-\estProbFlipOneToZero\right).\]
Note also that $\regressionFunction(x) =\left(\noisyRegressionFunction(x)-\probFlipZeroToOne\right)/\left(1-\probFlipZeroToOne-\probFlipOneToZero\right)$. Hence, by Lemma \ref{elementaryRatioLemma} for $x \in \goodDeltaSet$ we have,
\begin{align}\label{regEstimatorBoundInGoodDeltaSet}
\left|\est[\regressionFunction](x) - \regressionFunction(x)\right| \leq \frac{ \KConstant}{ 1-\probFlipZeroToOne-\probFlipOneToZero}  \cdot \max\left\lbrace \left( \frac{\epsDeltaN}{\densityFunction(x)} \right)^{\frac{\holderExponent}{2\holderExponent+\dimensionExponent}},\epsDeltaN^{\frac{\fastLimitsExponent \holderExponent}{\fastLimitsExponent(2\holderExponent+\dimensionExponent)+\holderExponent}} \right\rbrace.
\end{align}
Choose $\theta^0_*(n,\delta):=\min\{\tailsThreshold,\epsDeltaN^{\frac{\holderExponent}{\fastLimitsExponent(2\holderExponent+\dimensionExponent)+\holderExponent}}\}$ so that 
\begin{align*}
\epsDeltaN^{\frac{\fastLimitsExponent \holderExponent}{\fastLimitsExponent(2\holderExponent+\dimensionExponent)+\holderExponent}} \leq \left( {\epsDeltaN}/{\theta^0_*(n,\delta)} \right)^{\frac{\holderExponent}{2\holderExponent+\dimensionExponent}}\leq \tailsThreshold^{-\frac{\holderExponent}{2\holderExponent+\dimensionExponent}} \cdot \epsDeltaN^{\frac{\fastLimitsExponent \holderExponent}{\fastLimitsExponent(2\holderExponent+\dimensionExponent)+\holderExponent}}.
\end{align*}
Let $\theta \in \left(0,\theta_*^0(n,\delta)\right]$ be a parameter, whose value will be made precise shortly. We define $\goodDeltaSet^0:=\left\lbrace x \in \goodDeltaSet: \densityFunction(x) \geq \theta \right\rbrace$ and for each $j \geq 1$ we let \[
\goodDeltaSet^j:=\left\lbrace x \in \goodDeltaSet: 2^{1-j}\cdot \theta >\densityFunction(x) \geq 2^{-j}\cdot\theta \right\rbrace.\]
Since $\est[\decisionRule]_{n,\delta}(x)=\one\left\lbrace \est[\regressionFunction](x)\geq 1/2\right\rbrace$ and $\oracle(x)= \one\left\lbrace \regressionFunction(x)\geq 1/2\right\rbrace$ we see that for $x \in \goodDeltaSet^j$ with $\est[\decisionRule]_{n,\delta}(x) \neq \oracle(x)$,
\addtocounter{Kcounter}{-1}
\begin{align}
\left|\regressionFunction(x)-\frac{1}{2}\right|&\leq \left|\est[\regressionFunction](x) - \regressionFunction(x)\right| \nonumber\\
&\leq   \KConstant\cdot  \left(1-\probFlipZeroToOne-\probFlipOneToZero\right)^{-1}  \cdot \max\left\lbrace \left( \frac{2^j\cdot \epsDeltaN}{\theta} \right)^{\frac{\holderExponent}{2\holderExponent+\dimensionExponent}}, \epsDeltaN^{\frac{\fastLimitsExponent \holderExponent}{\fastLimitsExponent(2\holderExponent+\dimensionExponent)+\holderExponent}}\right\rbrace\nonumber \\ \addtocounter{Kcounter}{-1}
&\leq   \KConstant\cdot  \left(1-\probFlipZeroToOne-\probFlipOneToZero\right)^{-1}  \cdot \left( \frac{2^j\cdot \epsDeltaN}{\theta} \right)^{\frac{\holderExponent}{2\holderExponent+\dimensionExponent}}.\label{distanceFromHalfOnJSetWhenPredictingBadly}
\end{align}
The second inequality follows from (\ref{regEstimatorBoundInGoodDeltaSet}) combined with the definition of $\goodDeltaSet^j$ and the third inequality follows from the fact that $\theta  \leq 2^j\cdot \theta_*^0(n,\delta)$, so $\epsDeltaN^{\frac{\fastLimitsExponent \holderExponent}{\fastLimitsExponent(2\holderExponent+\dimensionExponent)+\holderExponent}} \leq \left( {2^j\cdot\epsDeltaN}/{ \theta} \right)^{\frac{\holderExponent}{2\holderExponent+\dimensionExponent}}$. Hence, by the margin assumption we have
\begin{align}\label{excessRiskOnGoodDeltaSet0}
\int_{\goodDeltaSet^0}\left|\regressionFunction(x)-\frac{1}{2}\right| d\marginalDistribution(x) \cdot \one\left\lbrace \est[\decisionRule]_{n,\delta}(x) \neq \oracle(x) \right\rbrace& \leq\frac{ \KConstant}{  \left(1-\probFlipZeroToOne-\probFlipOneToZero\right)^{1+\marginExponent}}  \cdot \left( \frac{\epsDeltaN}{\theta} \right)^{\frac{\holderExponent(1+\marginExponent)}{2\holderExponent+\dimensionExponent}}.
\end{align}
By the tail assumption, for $j\geq 1$ we have $\marginalDistribution\left(\goodDeltaSet^j\right) \leq \tailsConstant \cdot \left( 2^{1-j}\cdot\theta\right)^{\tailsExponent}$ and so
\begin{align}\label{excessRiskOnGoodDeltaSetJ}
\int_{\goodDeltaSet^j}\left|\regressionFunction(x)-\frac{1}{2}\right| \cdot \one\left\lbrace \est[\decisionRule]_{n,\delta}(x) \neq \oracle(x) \right\rbrace d\marginalDistribution(x)& \leq\frac{ \KConstant \cdot 2^{-j\left(\tailsExponent-\frac{\holderExponent}{2\holderExponent+\dimensionExponent}\right)}} { 1-\probFlipZeroToOne-\probFlipOneToZero} \cdot \theta^{\tailsExponent}\cdot \left( \frac{ \epsDeltaN}{\theta} \right)^{\frac{\holderExponent}{2\holderExponent+\dimensionExponent}}.
\end{align}
Combining (\ref{excessRiskOnGoodDeltaSet0}) and (\ref{excessRiskOnGoodDeltaSetJ}) with $\marginalDistribution\left(\X\backslash\goodDeltaSet\right)\leq \delta$ we see that
\begin{align}
&\risk\left(\est[\decisionRule]_{n,\delta} \right) - \risk\left( \oracle\right)\nonumber\\ & = 2 \int \left|\regressionFunction(x)-\frac{1}{2}\right| \cdot \one\left\lbrace \est[\decisionRule]_{n,\delta}(x) \neq \oracle(x) \right\rbrace d\marginalDistribution(x)\nonumber\\
& \leq  2\cdot \sum_{j = 0}^{\infty}\int_{\goodDeltaSet^j}\left|\regressionFunction(x)-\frac{1}{2}\right| \cdot \one\left\lbrace \est[\decisionRule]_{n,\delta}(x)\neq \oracle(x) \right\rbrace d\marginalDistribution(x)+\marginalDistribution\left(\X\backslash\goodDeltaSet\right)\nonumber \\
& \leq \frac{ \KConstant}{  \left(1-\probFlipZeroToOne-\probFlipOneToZero\right)^{1+\marginExponent}}  \cdot \left(\left( \frac{\epsDeltaN}{\theta} \right)^{\frac{\holderExponent(1+\marginExponent)}{2\holderExponent+\dimensionExponent}}+\theta^{\tailsExponent}\cdot \left( \frac{ \epsDeltaN}{\theta} \right)^{\frac{\holderExponent}{2\holderExponent+\dimensionExponent}}\right)+\delta,\label{tailsBalancedRiskUpperBoundThm}
\end{align}
where we used the assumption that $\tailsExponent > \holderExponent/(2\holderExponent+\dimensionExponent)$ so $\sum_{j=1}^{\infty} 2^{-j\left(\tailsExponent-\frac{\holderExponent}{2\holderExponent+\dimensionExponent}\right)}<\infty$. To complete the proof we define ${\theta}_*^1(n,\delta) = \epsDeltaN^{\frac{\marginExponent\holderExponent}{\tailsExponent(2\holderExponent+\dimensionExponent)+\marginExponent\holderExponent}} \in (0,1)$ so that the two terms in (\ref{tailsBalancedRiskUpperBoundThm}) are balanced. If $\theta_*^1(n,\delta)\leq \theta_*^0(n,\delta)$ then (\ref{tailsBalancedRiskUpperBoundThm}) holds with $\theta = {\theta}_*^1(n,\delta)$, which implies (\ref{conclusionInRiskBoundMaxFormulation}) If on the other hand $\theta_*^1(n,\delta)> \theta_*^0(n,\delta)$ then with $\theta = \theta_*^0(n,\delta)$, (\ref{tailsBalancedRiskUpperBoundThm}) holds and the term $\left( {\epsDeltaN}/{\theta} \right)^{\frac{\holderExponent(1+\marginExponent)}{2\holderExponent+\dimensionExponent}}$ dominates the $\theta^{\tailsExponent}\cdot \left({ \epsDeltaN}/{\theta} \right)^{\frac{\holderExponent}{2\holderExponent+\dimensionExponent}}$ term, which also implies (\ref{conclusionInRiskBoundMaxFormulation}). This completes the proof of (\ref{conclusionInRiskBoundMaxFormulation}) which implies Theorem \ref{mainUpperBoundThm}.

\QEDA

\section{Proof of the regression bound}\label{proofOfKNNRegressionBoundsAppendix}

In this section we prove Theorem \ref{knnRegressionBoundLepskiK}. We begin by proving the supporting Lemmas \ref{closeNeighboursLemma} \& \ref{knnEstIsCloseToItsConditionalExpectationLemma} which were also used in the proof of Theorem \ref{funcMaxThm}. We then prove Theorem \ref{knnRegressionBoundGeneralk}, a high probability bound for a deterministic $k$. We then deduce Theorem \ref{knnRegressionBoundLepskiK}.

\begin{lem}\label{closeNeighboursLemma} Suppose that $\marginalDistribution$ satisfies the minimal mass assumption with parameters $(\dimensionExponent, \densityFunction)$. Given any $n \in \N$, $\delta \in (0,1)$, $x \in \X$ and $k \in \N \cap [8 \log(1/\delta), \densityFunction(x) \cdot (n/2)]$, with probability at least $1-\delta$ over $\sample_f$ we have
$\dist\left(x,X_{\tau_{n,k}(x)}\right)< \left( {2k}/\left({\densityFunction(x) \cdot n}\right) \right)^{\frac{1}{\dimensionExponent}}$.
\end{lem}
The proof of Lemma \ref{closeNeighboursLemma} is similar to \cite[Lemma 8]{chaudhuri2014rates}.

\paragraph{Proof of Lemma \ref{closeNeighboursLemma}} By the minimal mass assumption combined with the fact that $k \leq \densityFunction(x) \cdot (n/2)$, if we take $r=(2k/(\densityFunction(x)\cdot n))^{\frac{1}{\dimensionExponent}}$ then we have $\marginalDistribution(B_r(x))\geq 2k/n$. Let $\Prob_{\bm{X}}$ denote the marginal distribution over $\bm{X}=\{X_i\}_{i \in [n]}$. Applying the multicative Chernoff bound we have
\begin{align*}
\Prob_{\bm{X}}\left[ \dist(x,X_{\tau_{n,k}(x)}) \geq r\right]
 & = \Prob_{\bm{X}}\left[ \sum_{i=1}^n\one\left\lbrace X_i \in B_r(x)\right\rbrace < k  \right]\\
 & \leq \Prob_{\bm{X}}\left[ \sum_{i=1}^n\one\left\lbrace X_i \in B_r(x)\right\rbrace < \frac{n}{2} \cdot \marginalDistribution(B_r(x))  \right]\\
 & \leq \exp\left(-\frac{n}{8}\cdot \marginalDistribution(B_r(x))\right) \leq \exp(-k/8)\leq \delta.
\end{align*}
\QEDA

 Let $\bm{X}=\left\lbrace X_i\right\rbrace_{i \in [n]}$, $\bm{Z}=\left\lbrace Z_i\right\rbrace_{i \in [n]}$ and $\Prob_{\bm{Z}|\bm{X}}$ denote the conditional probability over $\bm{Z}$, conditioned on $\bm{X}$, with $(X_i,Z_i) \sim \Prob_f$. 

\begin{lem}\label{knnEstIsCloseToItsConditionalExpectationLemma} For all $n \in \N$, $\delta \in (0,1)$, $x \in \X$, $\bm{X} \in \X^n$ and $k \in [n]$ we have,
\begin{align*}
\Prob_{\bm{Z}|\bm{X}}\left[
\left|  \hat{f}_{n,k}(x)- \frac{1}{k}\sum_{q \in [k]}f\left(X_{\tau_{n,q}(x)}\right)\right| \geq \sqrt{\frac{\log(2/\delta)}{2k}}\right] \leq \delta.
\end{align*}
\end{lem}

\paragraph{Proof} Note that $\hat{f}_{n,k}(x) = \frac{1}{k}\sum_{q \in [k]}Z_{\tau_{n,q}(x)}$ and the random variables $Z_{\tau_{n,q}(x)}$ are conditionally independent given $\bm{X}$. In addition, for each $q \in [k]$ we have $\E_{\bm{Z}|\bm{X}}\left[Z_{\tau_{n,q}(x)}\right] = f\left(X_{\tau_{n,q}(x)}\right)$. Hence, by Hoeffding's inequality we have
\begin{align*}
&\Prob_{\bm{Z}|\bm{X}}\left[
\left|  \hat{f}_{n,k}(x)- \frac{1}{k}\sum_{q \in [k]}f\left(X_{\tau_{n,q}(x)}\right)\right| \geq \sqrt{\frac{\log(2/\delta)}{2k}}\right]\\
&=
\Prob_{\bm{Z}|\bm{X}}\left[
\left| \frac{1}{k}\sum_{q \in [k]}Z_{\tau_{n,q}(x)}- \E_{\bm{Z}|\bm{X}}\left[\frac{1}{k}\sum_{q \in [k]}Z_{\tau_{n,q}(x)}\right]\right| \geq \sqrt{\frac{\log(2/\delta)}{2k}}\right] \leq \delta.    
\end{align*}

\QEDA

We have the following high probability performance bound.
\begin{theorem}\label{knnRegressionBoundGeneralk} Suppose that $f$ satisfies the H\"{o}lder assumption with parameters $(\holderExponent, \holderConstant)$ and $\marginalDistribution$ satisfies the minimal mass assumption with parameters $(\dimensionExponent, \densityFunction)$. Given any $n \in \N$, $\delta \in (0,1)$, $x \in \X$ and $k \in \N \cap [8 \log(2/\delta),\densityFunction(x) \cdot (n/2)]$, with probability at least $1-\delta$ over $\sample_f$ we have
\begin{align*}
\left| \hat{f}_{n,k}(x)-f(x) \right|< \sqrt{\frac{\log(4/\delta)}{2k}}+\holderConstant\cdot \left( \frac{2k}{\densityFunction(x) \cdot n} \right)^{\frac{\holderExponent}{\dimensionExponent}}.
\end{align*}
\end{theorem}
The proof of Theorem \ref{knnRegressionBoundGeneralk} is broadly similar to the proof of \cite[Theorem 1]{kpotufe2011k} adapted to our assumptions.

\paragraph{Proof of Theorem \ref{knnRegressionBoundGeneralk}} By Lemmas \ref{closeNeighboursLemma}, \ref{knnEstIsCloseToItsConditionalExpectationLemma} and the union bound, with probability at least $1-\delta$ over $\sample_f$, we have $\dist\left(x,X_{\tau_{n,k}(x)}\right)< \left( {2k}/\left({\densityFunction(x) \cdot n}\right) \right)^{\frac{1}{\dimensionExponent}}$ and 
\begin{align*}
\left|  \hat{f}_{n,k}(x)- \frac{1}{k}\sum_{q \in [k]}f\left(X_{\tau_{n,q}(x)}\right)\right| < \sqrt{\frac{\log(2/\delta)}{2k}}.
\end{align*}
By the H\"{o}lder assumption, combined with $\dist\left(x,X_{\tau_{n,q}(x)}\right) \leq \dist\left(x,X_{\tau_{n,k}(x)}\right)< \left( {2k}/\left({\densityFunction(x) \cdot n}\right) \right)^{\frac{1}{\dimensionExponent}}$, for $q \in [k]$ we have \[\left| f(X_{\tau_{n,q}(x)})-f(x)\right| \leq \holderConstant\cdot \left( \frac{2k}{\densityFunction(x) \cdot n} \right)^{\frac{\holderExponent}{\dimensionExponent}}.\]
Hence, the theorem follows by the triangle inequality.
\QEDA.

\paragraph{Proof of Theorem \ref{knnRegressionBoundLepskiK}} By Theorem \ref{knnRegressionBoundGeneralk} combined with the union bound we see that with probability at least $1-\delta$, the following holds simultaneously for all $k \in \N \cap [8 \log (2n/\delta), \densityFunction(x)\cdot n/2]$
\begin{align}\label{generalKRegApplicationEq}
\left| \hat{f}_{n,k}(x)-f(x) \right|&< \sqrt{\frac{\log(4n/\delta)}{2k}}+\holderConstant\cdot \left( \frac{2k}{\densityFunction(x) \cdot n} \right)^{\frac{\holderExponent}{\dimensionExponent}}.
\end{align}
We choose $\tilde{k} \in \N$ so maximally so that the first term in (\ref{generalKRegApplicationEq}) bounds the second,
\[\tilde{k}:= \left\lfloor \frac{1}{2} \cdot \left(\densityFunction(x)\cdot n\right)^{\frac{2\holderExponent}{2\holderExponent+\dimensionExponent}} \cdot \left( \frac{\log(4n/\delta)}{\holderConstant^2}\right)^{\frac{\dimensionExponent}{2\holderExponent+\dimensionExponent}}\right\rfloor.\]
We may assume without loss of generality that $ 8\log(2n/\delta)\leq \tilde{k}\leq \densityFunction(x)\cdot n/2$, since otherwise the RHS in (\ref{statementOfLepskiKNNIneq}) is trivial. Thus, we have
\begin{align}\label{tildeKBoundseq}
\frac{1}{4} \cdot \left(\densityFunction(x)\cdot n\right)^{\frac{2\holderExponent}{2\holderExponent+\dimensionExponent}} \cdot \left(\frac{ \log(4n/\delta)}{\holderConstant^2}\right)^{\frac{\dimensionExponent}{2\holderExponent+\dimensionExponent}}\leq \tilde{k} \leq \frac{1}{2} \cdot \left(\densityFunction(x)\cdot n\right)^{\frac{2\holderExponent}{2\holderExponent+\dimensionExponent}} \cdot \left(\frac{ \log(4n/\delta)}{\holderConstant^2}\right)^{\frac{\dimensionExponent}{2\holderExponent+\dimensionExponent}}.
\end{align} 
By (\ref{generalKRegApplicationEq}) combined with the upper bound in (\ref{tildeKBoundseq}) we see that for $q \in \N \cap [8 \log (2n/\delta), \tilde{k}]$ we have
\begin{align*}
\left| \hat{f}_{n,q}(x)-f(x) \right|&< \sqrt{\frac{\log(4n/\delta)}{2q}}+\holderConstant\cdot \left( \frac{2q}{\densityFunction(x) \cdot n} \right)^{\frac{\holderExponent}{\dimensionExponent}}\leq \sqrt{\frac{2\log(4n/\delta)}{q}}.
\end{align*}
Hence, $f(x) \in   \bigcap_{q \in \N \cap [8\log(2n/\delta),\tilde{k}]}   \hat{\mathcal{I}}_{n,q,\delta}(x) \neq \emptyset$, so $\tilde{k} \leq \hat{k}_{n,\delta}(x)$. Moreover, by the construction of $\hat{k}_{n,\delta}(x)$ we must have \[\hat{\mathcal{I}}_{n,\hat{k}_{n,\delta}(x),\delta}(x)\cap \hat{\mathcal{I}}_{n,\tilde{k},\delta}(x) \neq \emptyset.\]
Combining this with the fact that $\hat{f}_{n,\delta}(x) \in \hat{\mathcal{I}}_{n,\hat{k}_{n,\delta}(x),\delta}(x)$, $f(x) \in  \hat{\mathcal{I}}_{n,\tilde{k},\delta}(x)$ and each interval $ \hat{\mathcal{I}}_{n,q,\delta}(x)$ is of diameter $2\sqrt{2\log(4n/\delta)/q}$ we have
\begin{align*}
\left|\hat{f}_{n,\delta}(x)-f(x)\right| &\leq 2\sqrt{\frac{2\log(4n/\delta)}{\hat{k}_{n,\delta}(x)}}+2\sqrt{\frac{2\log(4n/\delta)}{\tilde{k}}} \\
&\leq (8 \sqrt{2}) \cdot \sqrt{\frac{\log(4n/\delta)}{\tilde{4k}}} \leq (8 \sqrt{2}) \cdot {\holderConstant}^{\frac{\dimensionExponent}{2\holderExponent+\dimensionExponent}} \cdot\left( \frac{\log(4n/\delta)}{\densityFunction(x)\cdot n} \right)^{\frac{\holderExponent}{2\holderExponent+\dimensionExponent}},
\end{align*}
where the final inequality follows from the lower bound in (\ref{tildeKBoundseq}).

\QEDA

\end{document}